\documentclass[10pt,reqno]{article}
\usepackage{a4wide}
\usepackage[english]{babel}

\usepackage{amssymb, amsfonts, amsbsy, latexsym}
\usepackage{amsmath}
\usepackage{amsthm}
\usepackage{graphicx}
\usepackage[colorlinks=true, allcolors=blue]{hyperref}

\usepackage{lipsum}
\usepackage{amsfonts}
\usepackage{graphicx}
\usepackage{epstopdf}

\usepackage[shortlabels]{enumitem}

\def\DD{\mathbf{\Delta}}
\def\bs{\boldsymbol}
\def\AA{\boldsymbol{A}}
\def\J{\mathcal{J}}
\def\L2L{L^2(\mathcal{J}) \rightarrow L^2(\mathcal{J})}

\newtheorem{theorem}{Theorem}[section]
\newtheorem{definition}[theorem]{Definition}
\newtheorem{proposition}[theorem]{Proposition}
\newtheorem{remark}[theorem]{Remark}
\newtheorem{lemma}[theorem]{Lemma}
\newtheorem{example}[theorem]{Example}

\newtheorem{approximation problem}[theorem]{Approximation problem}

\newtheorem{claim}[theorem]{Claim}

\title{Transferability of Graph Neural Networks: an Extended Graphon Approach}

\date{}

\newcommand*{\email}[1]{\href{mailto:#1}{\nolinkurl{#1}} } 

\author{Sohir Maskey\thanks{Department of Mathematics, LMU Munich, 80333 Munich, Germany
  (\email{maskey@math.lmu.de} , \email{kutyniok@math.lmu.de}).}
\and Ron Levie\thanks{Faculty of Mathematics, Technion - Israel Institute of Technology
  ( 
  \email{levieron@technion.ac.il}). }
\and Gitta Kutyniok\footnotemark[1] \thanks{Department of Physics and Technology, University of Tromsø, 9019 Tromsø, Norway}
  }

\begin{document}

\maketitle

\begin{abstract}
We study spectral graph convolutional neural networks (GCNNs), where filters are defined as continuous functions of the graph shift operator (GSO) through functional calculus. 
 A spectral GCNN is not tailored to one specific graph and can be transferred between different graphs.
It is hence important to study the \emph{GCNN transferability}:
the capacity of the network to have approximately the same repercussion on different graphs that represent the same phenomenon. Transferability ensures that GCNNs trained on certain graphs generalize if the graphs in the test set represent the same phenomena as the graphs in the training set. 

In this paper, we consider a model of transferability based on graphon analysis. Graphons are limit objects of graphs, and, in the graph paradigm, two graphs represent the same phenomenon if both approximate the same graphon.
Our main contributions can be summarized as follows: 1) we prove that any fixed GCNN with continuous filters is transferable under graphs that approximate the same graphon, 2) we prove transferability for graphs that approximate  unbounded graphon shift operators, which are defined in this paper,  and, 
3) we obtain non-asymptotic approximation results, proving linear stability of GCNNs. This extends current state-of-the-art results which show asymptotic transferability for polynomial filters under graphs that approximate bounded graphons.

\end{abstract}



\section{Introduction}
\label{intro}
CNNs~\cite{LecunCNNintro} have shown great performances in a variety of tasks on data defined on Euclidean domains. 
In recent years, there has been a growing interest in machine learning on non-Euclidean data, and especially,  data represented by graphs, leading to 
graph convolutional neural networks (GCNNs). 
%
Typically, GCNNs fall under two categories: spatial and spectral methods. Spatial-based approaches define graph convolutions by information propagation in the spatial domain. Spectral-based methods define convolutions in the frequency domain, taking inspiration from the classical Euclidean convolution theorem as follows.  First, a frequency domain is defined for graphs by generalizing the Fourier transform on them. Then,  graph convolutions are defined as element-wise multiplication in the graph frequency domain. In this paper, we study the \textit{functional calculus approach} \cite{defferrard2017convolutional, kipf2017semisupervised, levie2018cayleynets, armaFilters} for filtering, which bridges the gap between the spatial and spectral  methods.
 Even though the functional calculus GCNNs are defined as spectral methods, they can be implemented directly in the spatial domain by  applying filters on the graph shift operator (GSO).
 For a survey on spectral and spatial graph neural networks, we refer to~\cite{Bronstein_2017} and~\cite{Wu_2021}.

In the functional calculus approach, the filters of the GCNN are defined via scalar functions $h:\mathbb{R}\rightarrow\mathbb{R}$ applied on the GSO $\DD$ via functional calculus, namely, as $h(\DD)$ (see Subsection~\ref{subsection:GraphFilters} for more details). The functions $h$ are parameterized by some trainable parameters. For example, when $h$ are polynomials or rational functions, the trainable parameters are the coefficients of $h$~\cite{defferrard2017convolutional, levie2018cayleynets}.  Any GCNN is defined by its underlying filters $h$, which are not linked to a specific graph. This means that an individual GCNN can be used to process any signal on any graph by plugging different GSOs $\DD$, corresponding to different graphs, in the filters $h$ of the GCNN. 

 In \emph{multi-graph settings}, the datasets consist of a set of signals defined on different graphs.  
 Since the test set may consist of different graphs than the training set, the GCNN should generalize to graphs unseen during training. To guarantee a good generalization capability, it is desirable to show that the GCNN is \emph{transferable} between graphs. In~\cite{LevieTransfSpectralGraphFiltersLong}, the notion of GCNN transferability was described as follows: a GCNN is transferable, if, whenever two graphs represent the same phenomenon, the GCNN has approximately the same repercussion on both graphs.
 Transferability promotes generalization, since  we may encounter  a graph that realizes some phenomenon in the training set, and then encounter a different graph representing the same phenomenon in the test set. Having learned how to cope with the first realization of the phenomenon, the network handles its other versions well. For example, in a dataset of mesh graphs, surfaces are the underlying phenomena. If a filter is learned on one mesh-graph, discretizing a surface, it is desirable to show that the same filter has approximately the same effect on a different mesh-graph, discretizing the same surface.

 Transferability is analyzed formally using a modeling paradigm that was introduced in~\cite{LevieTransfSpectralGraphFiltersLong}. First, we require a mathematical model of ``two graphs representing the same phenomenon.'' Second, a notion that a fixed GCNN ``has approximately the same repercussion on two graphs'' is needed. Equipped with two such mathematical definitions, one can formulate rigorously and prove transferability.  
  Some examples of transferability frameworks are the stability approach~\cite{LevieTransfSpectralGraphFiltersShort, kenlayPolSpecFilters, kenlay2021interpretable, kenlay2021stability, Gama_2020}, the sampling approach~\cite{LevieTransfSpectralGraphFiltersLong, keriven2020convergence}, and the graphon approach~\cite{RuizI}. Each framework differs in the way it formalizes the transferability paradigm. In Subsection~\ref{subsec:survey} we survey transferability frameworks from the literature. In this work, we extend and generalize the graphon transferability approach.

\subsection{Comparison with Other Works}
\label{subsec:comparison}
The approach we consider in this paper for modelling graphs that represent the same phenomenon is called the \emph{graphon approach}. \emph{Graphons}~\cite{lovasz2012} are limit objects of sequences of weighted graphs. In some sense, a graphon can be thought of as an adjacency matrix, or GSO, of a graph with a continuous node index set. The continuous node index set is typically taken to be $[0,1]$, and a graphon is formally a function $W:[0,1]^2\rightarrow[0,1]$.  The value $W(x,y)$ is interpreted as the probability of having an edge between the nodes $x$ and $y$, or as the weight of the edge $(x,y)$. 
To model graphs that represent the same phenomenon, the graphon approach relies on the notion of homomorphism density. Given a simple graph $F$, informally, the \emph{homomorphism density} $t(F,G)$  is the probability that $F$ occurs as a substructure of the graph $G$ (see Definition~\ref{def:graphHomDensity}  for a precise formulation). Similarly, the homomorphism density $t(F,W)$ is also defined for graphons (see Definition~\ref{def:homdensityGraphon} and \cite{lovasz2012}). In the graphon approach, two graphs $G_1$ and $G_2$ represent the same phenomenon if for every simple graph $F$ we have $t(F,G_1)\approx t(F,G_2)$. The sequence of graphs $(G_n)_n$ converges to the graphon $W$ if  for every simple graph $F$, $t(F,G_n)\rightarrow t(F,W)$.  

In this subsection, we compare our transferability results with other related transferability results \cite{keriven2020convergence, 9362328, 9287695,RuizI, ruiz2021graph}, which also use the graphon approach. 
The analysis of transferability of GCNNs in the graphon approach in the papers \cite{keriven2020convergence, 9362328,9287695, RuizI, ruiz2021graph}  consists of two steps: 1) given a sequence of graphs $G_n$ and a graphon $W$,  first prove that  $G_n$ converges to $W$, and 2) given a filter $h$ and a sequence of graphs $G_n$ converging to a graphon $W$, prove that the graph filter $h(G_n)$ converges to the graphon filter $h(W)$ (see Definition  \ref{def:graphFilter} for more detail).  The first step is treated already in other works, e.g., \cite{Borgs2007graph,lovasz2012,  graphonsampling_klopp} show that sampling random graphs from a graphon with increasing number of nodes leads to convergence, and   \cite{ruiz2021graph} shows that graphs that are sampled in a regular grid from the graphon converge to the graphon.  Hence, in this paper we only focus on the second step of the analysis.


We have already seen that analyzing transferability in the graphon approach consists of performing an approximation analysis. For this, one can consider various types of assumptions on the object of study. We call \emph{topological assumptions} any construction that depends on the topology or metric of the underlying space. For example, assuming that a function $f:[0,1]^2 \to \mathbb{R}$ is Lipschitz means that we treat the space $[0,1]$ as a metric space. A \emph{measure theoretic assumption} treats the underlying space as a measure space. For example, assuming that a function $f:[0,1]^2\to \mathbb{R}$ is only measurable or square-integrable implies that we treat the space $[0,1]$ only as a measure space, not as a topological. 
Since there is a natural measure space associated with topological spaces, namely its Borel space, 
 assuming only measure theoretic assumption gives more general results than topological assumptions.

To analyze the transferability of GCNN filters in the graphon approach it is important to not assume any topological constraints on the graphon. By this, the graphon approach relies purely on measure-theoretical objects defined on the atom-free standard probability space $[0,1]$ equipped with the Lebesgue measure.  
A well-known fact is that every atom-free standard space is isomorphic up to a measure-preserving isomorphism.  Standard probability spaces cover a large class of possible spaces, e.g., $\mathbb{R}^n, \mathbb{C}^n$, every Lie group with countable components, every Polish space and standard Borel space with atom-free  probability measure. In particular, since there is no notion of dimension in measure theory, the space $[0,1]$ is equivalent as a probability space to a rich class of probability spaces over topological spaces of any dimension. Hence, if   no topological constraints are assumed for the graphon and the graphon signal, the transferability results  hold for all atom-free standard  probability space.  


The related works \cite{9362328, 9287695,RuizI} consider only measure-theoretical assumptions on the graphon but show transferability without giving rates of converges. The papers  \cite{keriven2020convergence, ruiz2021graph} show transferability with  rates of converges, but impose strong topological constraints on the graphon.
Our work is the first to derive transferability bounds with rate of convergence, while treating graphon purely in a measure theoretic setting, which makes our results very general.

In the following, we recall in more detail the related works \cite{9362328, 9287695,RuizI, keriven2020convergence, ruiz2020graphon, ruiz2021graph} and
  compare them  to our results.
The authors  in \cite{9362328,9287695, RuizI} consider simple graphs that are represented by their adjacency matrix. It is shown that polynomial filters are asymptotically transferable  for a sequence of graphs that converges to a graphon with weights in $[0,1]$. The assumptions on the considered graphons are purely measure-theoretical, and thus generalize to any atom-free standard probability space. However, the approach in \cite{9362328,9287695, RuizI} does not give any  rates of convergence for the transferability results. They are moreover restricted only to polynomial filters, while filters in practice are often non-polynomial  \cite{levie2018cayleynets, armaFilters}. We generalize these results by allowing graphs with arbitrary GSOs, having possibly negative weights, which includes  graph Laplacians. We also allow general unbounded graphons, which we define in this paper in Definition \ref{def:unboundedGSRO}, and any  continuous filter (not just polynomials). Furthermore, we obtain convergence rates while only considering a pure measure theoretic construction (no topological constraints).

Too derive a rate of convergence,  \cite{keriven2020convergence,ruiz2020graphon, ruiz2021graph}  consider    a framework in which a sequence of graphs is sampled from a template graphon. The authors show  that highly regular filters (polynomials or analytic functions) applied to such a sequence of graphs converge to the filter applied to  template graphon. This result   holds only under strong topological constraints on the limit graphon and the graphon signal (e.g., Lipschitz continuity). Thus, the   transferability results (that include rates of convergence) in \cite{keriven2020convergence,ruiz2020graphon, ruiz2021graph} do not generalize to all atom-free standard probability spaces and to general measurable graphon (that can for example be discontinuous everywhere). 

In contrast, we prove transferability with rates of convergence  while assuming no  topological constraints on the graphon and on the graphon signal. In particular, our theory then applies to underlying spaces of any dimension and  irregular graphons that need only be measurable  and can even be discontinuous everywhere.  While other papers use basic linear algebra to derive their results,  we generalize and improve the theory by using tools from graphon theory and functional analysis, in particular, functional calculus and  operator Lipschitz functions \cite{OLfcts}. 
  


\subsection{Survey of GCNN Transferability Approaches} 
\label{subsec:survey}
In this subsection, we survey different approaches for defining ``two graphs represent the same phenomenon'' and a fixed filter having ``approximately the same repercussion'' on two graphs.
\subsubsection*{Stability Under Small Graph Perturbations} 
This framework allows studying the transferability of graphs sharing the same node set. Here, two graphs represent the same phenomenon if their GSOs are small  perturbations  of each other. The notion of ``small perturbation'' is modeled differently in different works. Moreover, the change in repercussion of the filter due to the perturbation of the GSO is also defined in various ways.
 The first work that considered this framework was  by Levie et al. in~\cite{LevieTransfSpectralGraphFiltersShort}. There, the perturbation of the GSO is defined in the induced $l_2$-norm sense, between the GSO $\DD_1$ and its perturbed version $\DD_2$. The repercussion error of the filter was also defined by the induced $l_2$-norm.
 Given a function $h$ from the so-called Cayley smoothness class, and assuming that $\|\DD_1 - \DD_2\|_2 \leq 1$, the following stability result was proven
\begin{equation}
\label{eq:LevieContribution}
\|h(\DD_1) - h(\DD_2)\|_2 \leq C\|\DD_1 - \DD_2\|_2,
\end{equation}
where the constant $C>0$ depends on $h$ and $\|\DD_1\|_2$.

Similar results were also proven in~\cite{kenlay2021interpretable}, where graphs are assumed to be simple. The authors show that polynomial and low-pass filters are stable to arbitrary perturbations, in the sense that $\|\DD_1 - \DD_2\|_2\leq 1$ does not need to be assumed in~(\ref{eq:LevieContribution}).
Furthermore, if the structural change of the graph consists of adding and deleting edges, the authors constrain the GSO error $\|\DD_1 - \DD_2\|_2$  by interpretable properties directly related to this structural change.
In~\cite{Gama_2020}, Gama et al. consider weighted graphs with GSOs. The authors showed stability of Lipschitz continuous polynomial filters in a similar setting as~(\ref{eq:LevieContribution}). The difference with respect to the previous papers is that the measure of graph perturbation is the norm distance, minimized over all permutations of the node set of $\DD_2$.

Closely related to stability of GCNNs,~\cite{ NEURIPS2019_e2f374c3, pmlr-v119-bojchevski20a} also studied the robustness of some GCNNs on simple graphs. The authors provided certificates that guarantee that GCNNs are stable in a semi-supervised graph or node classification task. The authors guarantee that the classification of nodes and graphs remains correct  while structurally perturbing the input graph or perturbing the nodes' input features. The structural perturbations are the deletion or addition of edges.
The repercussion error of a fixed GCNN is defined either pointwise on a single node  of interest, or the classification outcome for the whole graph.
The definition of the GCNN in these papers was restricted to a subclass of spectral methods.

\subsubsection*{Transferability Under Sampling From Latent Spaces}
A major drawback of analyzing transferability of GCNNs under GSO perturbations is the assumption that the  graphs share the same node set.
In contrast, in the sampling approach to transferability,  graphs are assumed to represent the same phenomenon if their nodes are sampled from the same underlying latent space, which is typically a metric, topological, or measure space $\mathcal{M}$. This allows comparing graphs of arbitrary sizes and topologies. In~\cite{LevieTransfSpectralGraphFiltersLong}, Levie et al. introduce a framework where graph nodes and GSOs are sampled from a ``continuous'' underlying model. The model consists of a measure space $\mathcal{M}$ and a continuous Laplacian $\mathcal{L}$, which is a self-adjoint operator with discrete spectrum. The repercussion of a filter is treated in the ``sampling--interpolation'' approach. It is assumed that two latent sampling operators $S_1$ and $S_2$ exist for the two graphs $G_1$ and $G_2$, which describe how the graph data is generated from the latent space. For example, the nodes of the graph can be assumed to belong to a latent space, and sampling evaluates signals defined over the latent space on these sample nodes, to generate a graph signal. Conversely, two interpolation operators $R_1$ and $R_2$ take  graph signals and return signals over the topological space. For a fixed filter $h$, the transferability error of the filter on graph signals sampled from a latent continuous signal $\psi \in L^2(\mathcal{M})$ is  defined as $ \|R_1h(\DD_1)S_1\psi - R_2h(\DD_2)S_2\psi \|_{L^2(\mathcal{M})}$. It is shown that the transferability error of the filter is bounded by the transferability of the Laplacian and the consistency error, i.e.,
\begin{equation}
\begin{aligned}
\label{eq:LevieTransferabilityError}
 \|R_1h(\DD_1)S_1\psi - R_2h(\DD_2)S_2\psi \|_{L^2(\mathcal{M})} & \leq C_1 \sum_{j=1,2}\| \mathcal{L}\psi - R_j \DD_j S_j \psi\|_{L^2(\mathcal{M})}  \\ & + C_2\sum_{j=1,2}\| \psi - R_jS_j \psi\|_{L^2(\mathcal{M})}.
\end{aligned}
\end{equation}
Based on this inequality, the authors develop a transferability inequality for end-to-end  GCNNs. Moreover, it was shown that the   right-hand-side  of~(\ref{eq:LevieTransferabilityError}) is small when the graphs are sampled randomly from the latent space, hence proving transferability for randomly generated graphs.  

A related approach was proposed in~\cite{keriven2020convergence}, where the authors   consider   random graph models. A random graph model is a tuple of a compact metric space equipped with a probability measure and a ``continuous'' shift operator, defined as a graphon. The probability measure is used to sample nodes, and the graphon is used to sample edges, to generate a graph from the continuous model.
The stability of analytic filters and GCNNs to small deformations of the underlying random graph model was proved. 

\subsubsection*{Transferability Between Graphs Converging to the Same Graphon}  
\label{subsubsec:transf}
 Ruiz et al. first used the graphon framework in~\cite{RuizI} to define a transferability analysis of graph filters. 
 In that work, the repercussion of a polynomial filter $h$ is measured in an asymptotic sense, using the notion of signal induction.
 For a signal $\mathbf{x} \in \mathbb{R}^n$, its induced graphon signal $\psi_\mathbf{x} \in L^2[0,1]$ is a piecewise constant representation of $\mathbf{x}$ over the continuous index set $[0,1]$ (see Definition~\ref{definition:inducedsignal}). Furthermore, a graphon $W$ induces the integral operator $T_W f (u) := \int_0^1 W(u,v) f(v) dv$.
 The result is  given in  \cite[Theorem 4]{RuizI}    and can be stated informally as follows.

\begin{theorem}[informal]
\label{thm:informal:Ruiz}
Let $h$ be a polynomial Lipschitz continuous filter and $(G_n)_n$ be a sequence of graphs with adjacency matrices $(\bs{A}_n)_n$. Suppose that $G_n$ converges in the homomorphism density sense to the graphon $W$. Let $\mathbf{x}_n \in \mathbb{R}^n$ be a sequence of graph signals  converging to the graphon signal $y \in L^2[0,1]$ in the sense that the induced signals $\psi_{\mathbf{x}_n}$ converge to $y$, i.e., $\|\psi_{\mathbf{x}_n} - y\|_{L^2[0,1]} \xrightarrow{n \to \infty} 0$.
Denote $\mathbf{z}_n = h(\bs{A}_n/n)\mathbf{x}_n$
for $n \in \mathbb{N}$. Then, the output of the graph filters converges in the induction sense to the output of the graphon filter, namely,
$$
\|\psi_{ \mathbf{z} _n} - h(T_W)y\|_{L^2[0,1]} \xrightarrow{n \to \infty} 0.
$$
\end{theorem}

\cite{9362328, keriven2020convergence, ruiz2021graph}  consider transferability of GCNNs under sampling from a graphon $W$. The authors consider a sampling procedure leading to  a graph sequence $G_n$ that converges to $W$ in the homomorphism density sense (see Definition \ref{def:graphConvToGraphon}). The main results then show, under topological restrictions on the graphon, the graphon signal and the filter, that the output of the graph filters converge in the induction sense to the output of the graphon filter, similarly to Theorem \ref{thm:informal:Ruiz}.

 As described in Subsection \ref{subsec:comparison}, related transferability results in the graphon approach are either purely asymptotic while generalizing to other standard probability spaces, or a rate of convergence is calculated, but the results do not generalize to other standard probability spaces since the graphon is assumed to be a Lipschitz continuous function. 

\subsection{ Summary of our Main Contribution}
\label{subsec:contributions}


We consider in this paper dense graphs, namely, graphs in which the number of edges is of order of the number of nodes squared. Graphons are the natural limit object of such graphs. 
\begin{definition}
\label{Def:GWM}
Given a graph $G$ with $n$ nodes and GSO $\DD$, we define the \emph{graph weight matrix (GWM)} $\bs{A}$ as the matrix
\begin{equation}
    \bs{A} = n \DD .
    \label{eq:GWM}
\end{equation} 
\end{definition}

 The relation (\ref{eq:GWM}) is interpreted as follows.
 GSOs $\DD$ define how graphs operate on graph signals. To accommodate the transferability between dense graphs of different sizes, GSOs are assumed to be roughly normalized in the induced $l^2$-norm. On the other hand, GWMs describe the probability of each edge in the graph to exist, in case the entries of the GWM are in $[0,1]$. Accordingly, GWMs are roughly normalized in the entry-wise infinity norm. This motivates Definition \ref{Def:GWM}.

The way to compare how graphs of different sizes operate on signals in the graphon approach is by \emph{inducing} their GWMs to graphons. 
Given a GWM $\bs{A}$, its \emph{induced graphon}  $W_{\bs{A}}$ is a piecewise continuous representation of the matrix $\bs{A}$ over the continuous index set  [0,1]~\cite{lovasz2012}.  
The \emph{graphon shift operator (GRSO)} corresponding to the graphon $W$ is defined as the integral operator $T_{W} f (u) := \int_0^1 W(u,v) f(v) dv$. GRSOs are the limit objects of GSOs.
  Since filters $h(\DD)$ are GSOs themselves, and since GSOs and GWMs are related by (\ref{eq:GWM}), the notation $W_{nh(\DD)}$ means the graphon induced by the GWM corresponding to the GSO $h(\DD)$. Hence, given two graphs of size $n_1$ and $n_2$, the way to compare $h(\DD_1)$ to $h(\DD_2)$ is by
\[\|  T_{W_{n_1h(\DD_1)}} - T_{W_{n_2h(\DD_2)}}\|_{L^2[0,1]\rightarrow L^2[0,1]}. \]

In this work we generalize Theorem~\ref{thm:informal:Ruiz} in~\cite{RuizI} by allowing graphs with arbitrary GSOs, general unbounded graphon shift operators with values in $\mathbb{R}$, and general continuous filters.  Another important contribution is that we derive non-asymptotic transferability bounds without any topological assumptions.   We moreover consider end-to-end graph convolutional networks. 
We summarize our contributions as follows. 

\subsubsection*{Asymptotic Transferability}

A common approach   in data science  for studying and comparing  graphs is counting motifs, which are  substructures of the graphs, namely, simple graphs (see, e.g., \cite[Section 2.1.1]{hamilton2020graph}). Since homomorphism densities can be interpreted as the density of motifs in a graph (see Subsection \ref{subsec:homnumbers}), we choose the notion of homomorphism densities to model graphs that represent the same phenomenon.   On the other hand, when doing signal processing, graph filters act on graph signals as operators.  Thus, when comparing the repercussion of a filter realized on different graphs, we consider the $L^2$-induced operator norm of the induced GRSOs and the $L^2$-norm of the induced graphon signals. 
 Since our transferability framework is built on these two definitions, namely, homomorphism densities and $L^2$-induced operator norms,  we note that they are closely related: convergence in homomorphism density is equivalent (up to relabelling of the graphs' nodes) to convergence in the $L^2$-induced operator norm (see Theorem \ref{thm:cutnormconv2}).  

We prove asymptotic transferability for continuous filters in the following sense, that we formulate here informally, and write rigorously in Theorem~\ref{cor:DiscrFilterConv}.
\begin{theorem}[informal]
\label{informal:thm1}
Let $(G_n)_n$ be a sequence of graphs with GSOs $(\DD_n)_n$. Suppose that $G_n$ converges to a graphon   $W: [0,1] \times [0,1] \rightarrow \mathbb{R}$ in the homomorphism density sense (see Definition~\ref{def:graphConvToGraphon}). Let $h:\mathbb{R}\rightarrow\mathbb{R}$ be a continuous filter with $h(0)=0$. Then the, by $h(\mathbf{\Delta}_n)$ induced, integral operators $(T_{W_{h(\DD_ n)}})_n$ converge to the graphon filter $h(T_W)$ in operator norm, namely,
$$ 
\|  T_{W_{nh(\DD_ n)}} - h(T_W)\|_{L^2[0,1]\rightarrow L^2[0,1]} \xrightarrow{n \to \infty} 0. 
$$
\end{theorem}

 In fact, Theorem \ref{informal:thm1} with $h(x)=x$ shows that convergence of graphs in homomorphism density implies  convergence of the induced GRSOs in operator norm. 
Theorem \ref{informal:thm1} can be interpreted as a tranferability result between graphs as follows. 
Let $\varepsilon > 0$ be a desired error tolerance  for the repercussion error of the filter $h$. There exists a large enough $N$ such that for every $m,n>N$, by the triangle inequality, 
$$ 
\|  T_{W_{nh(\DD_n)}} -T_{W_{mh(\DD_m)}}\|_{L^2[0,1]\rightarrow L^2[0,1]} < \varepsilon. 
$$
Here, we interpret $\DD_n$ and $\DD_m$ as a pair of arbitrary GSOs that represent the same phenomenon in the homomorphism density sense.

\subsubsection*{Approximation Rate}
Theorem~\ref{informal:thm1} shows asymptotic transferability in a very general setting, but with no convergence rate. In Theorem~\ref{informal:thm2} we present non-asymptotic transferability bounds under some assumption on the filter. 
We give the following informal result, which we formulate rigorously in Proposition~\ref{convRateC1filter}. 
\begin{theorem}[informal]
\label{informal:thm2}
  Let $h \in C^1$ be such that $h'$ is Lipschitz and $h(0)=0$. Suppose that $G_1$ and $G_2$ are graphs with $n_1$ and $n_2$ nodes respectively, 
  represented by the GSOs $\mathbf{\Delta}_1$ and $\mathbf{\Delta}_2$. Let $\AA_1$ and $\AA_2$ be the GWMs corresponding to $\DD_1$ and $\DD_2$. Then
$$
\|T_{W_{n_1h(\DD_1)}} -  T_{W_{n_2h(\DD_2)}}\|_{L^2[0,1]\rightarrow L^2[0,1]} \leq C\| T_{W_{\AA_1}} -  T_{W_{\AA_2}}  \|_{L^2[0,1]\rightarrow L^2[0,1]},  
$$
where $C \in \mathbb{N}$ is a constant that depends on the filter $h$. 
\end{theorem}

 We mention that in our setting the graphs do not need to be of the same size and can be represented by arbitrary symmetric matrices. Furthermore, we allow more general filters than polynomials.


\subsubsection*{Graphs Approximating Unbounded Graphon Shift Operators}
\hfill \\ 
GRSOs cannot represent all important limit objects of graph shift operators, e.g., Laplace-Beltrami operators. To allow such examples, 
we define 
\emph{unbounded graphon shift operators (unbounded GRSOs)} --  self-adjoint unbounded operators that become GRSOs when restricted to their Paley-Wiener spaces.
As opposed to the similar approach in~\cite{LevieTransfSpectralGraphFiltersLong}, we allow an accumulation point at $0$ in the spectrum of the unbounded GRSO, which allows treating a richer family of operators. 
The following theorem is an informal version of Theorem~\ref{prop:convunbdd}. 
\begin{theorem}[informal]
Let $h: \mathbb{R} \rightarrow \mathbb{R}$ be a function with additional assumptions given in Theorem \ref{prop:convunbdd}. Let $(G_n)_n$  be a sequence of graphs with GSOs $(\mathbf{\Delta}_n)_n$, converging to an unbounded GRSO $\mathcal{L}$. Then, for every $\lambda > 0$
$$ 
\| P_\mathcal{L}(\lambda) T_{W_{nh(\mathbf{\Delta}_{n})}} P_\mathcal{L}(\lambda) - h(\mathcal{L} )\|_{L^2[0,1] \to L^2[0,1]} \xrightarrow{n \to \infty} 0, 
 $$ 
where $P_\mathcal{L}(\lambda)$ is the projection onto the space of band limited signals on the band $[-\lambda,\lambda ]$, also called the Paley-Wiener space associated with $\mathcal{L}$ (see Definition \ref{def:PWspace}). 
\end{theorem}
\subsection{Overview}
 In Section~\ref{transfGCNN} we give an  introduction to graph shift operators and graph filters. Section~\ref{sec:ClassicalGraphons} starts by introducing convergence for sequences of graphs. Then, graphons and graphon shift operators, limit objects of convergent sequences of graphs and GSOs, are discussed. The notation and most results cited there are from~\cite{lovasz}. 
 We introduce our transferability framework and the framework to analyze unbounded GRSOs in Section \ref{sec:transframework}.
 In Section~\ref{sec:main}, we present the main results about transferability of graph filters and GCNNs. In the Appendix, we provide proofs and theoretical backgrounds. 
The foundations of graphons are given in~\ref{Appendix:Graphons}.

\section{Graph Convolutional Neural Networks (GCNN)}
\label{transfGCNN}
In this section, we recall basic definitions of graphs with GSOs and the definition of functional calculus filters.

\subsection{Graph Shift Operators} 
\label{subsec:graphLaplacian}
There are various ways to represent the structure of a graph $G$. In this work, we consider representing $G$ by symmetric matrices  $ \DD \in \mathbb{R}^{n \times n}$, which are called \emph{graph shift operators (GSOs)} or \emph{graph Laplacians}.
 Accordingly, we use the following definition of a graph.
 \begin{definition}
A graph $G$ with GSO $\DD$ is a triplet $G = (V, E, \bs{\DD})$ where  $V=\{1,\ldots,n\}$ is the node set, $E \subset V \times V$ is the edge set, and $\DD\in\mathbb{R}^{n\times n}$ is a symmetric matrix called the GSO.
 \end{definition}

 In practice, graphs in datasets are given  by an adjacency matrix $\mathbf{W}$, with entries in $[0,1]$ if the graph is weighted, and in $\{0,1\}$ if it is unweighted. In practical deep learning, one typically modifies the ``raw data'' adjacency matrix $\mathbf{W}$ and constructs a GSO $\DD$ with entries in $\mathbb{R}$. Some common choices for $\DD$ are the combinatorial and the symmetric graph Laplacians or the graphs' adjacency matrix itself. The filters are then defined with respect to $\DD$.  In our analysis, we assume that the graph is already given with the GSO $\DD$,  no matter how it was constructed from the ``raw data'' $\mathbf{W}$.  The graph weight matrix GWM, not to be confused with the raw GSO, is defined directly from $\DD$ as $\bs{A}=n\DD$ (see Definition~\ref{Def:GWM}).

\subsection{ Graph Filters}
\label{subsection:GraphFilters}
In this subsection, we recall the definition of \emph{functional calculus filters}, as described in~\cite{LevieTransfSpectralGraphFiltersLong}.    For a graph $G$ with GSO $\DD \in \mathbb{R}^{n \times n}$, a \emph{graph signal} $\mathbf{x} = (x_i)_{i=1}^n \in \mathbb{C}^n$  is a vector where $x_i$ is called the feature of the $i$'th node.
 Given $\mathbf{x}, \mathbf{y} \in  \mathbb{C}^n$, the inner product between $\mathbf{x}$ and $\mathbf{y}$ is defined to be 
 \[ \langle \mathbf{x}, \mathbf{y} \rangle_{\mathbb{C}^n} =  \sum_{i=1}^n x_i \bar{y}_i.\]
 \begin{definition}
 \label{def:graphFilter}
 Let $G$ be a graph with a GSO $\DD$, and let $\{ \lambda_j, \phi_j\}_{j=1}^n$ be the eigenvalues and eigenvectors of $\DD$.
 \begin{itemize}
    \item[$(1)$] A \emph{filter} is a  continuous function $h:\mathbb{R}\rightarrow\mathbb{R}$.  
     \item[$(2)$] The \emph{realization of the filter} $h(\DD)$ is defined to be the operator
 \begin{equation}
\label{eq:funccalcfilter}
 h(\DD) \mathbf{x} = \sum_{j=1}^n h(\lambda_j) \langle \mathbf{x}, \phi_j  \rangle \phi_j,   
\end{equation}
applied on any graph signal $\mathbf{x}  \in \mathbb{C}^n$. 
 \end{itemize}  
 \end{definition}

 For brevity, we often call $h(\DD)$ simply a filter.
As an example, for a rational function 
\[  h(\lambda) = \frac{\sum_{n=0}^N h_n\lambda^n }{\sum_{m=0}^M k_m\lambda^{m}}  ,\] 
it is easy to see that
\begin{equation}
\label{eq:rationalFilters}
     h(\DD) = \left( \sum_{n=0}^N h_n \DD^n \right) \cdot  \left(   \sum_{m=0}^M k_m \DD^{m} \right)^{-1} .
\end{equation}
Thus, the polynomial filters in~\cite{kipf2017semisupervised, defferrard2017convolutional} and rational filters in~\cite{levie2018cayleynets, armaFilters} are  examples of functional calculus filters~(\ref{eq:funccalcfilter}).

\subsection{Spectral Graph Convolutional Neural Networks}
\label{subsec:GCNN}
Similarly to CNNs, a GCNN is defined as a deep convolutional architecture with three main components per layer: a set of graph filters, a non-linear activation function and optionally pooling. In this work, we study architectures without pooling. Example applications are semi-supervised node classification and node regression. 

 GCNNs are defined by realizing a \emph{spectral convolutional neural network (SCNN)} on a graph. SCNNs are defined independently of a particular graph topology and can be even applied on non-graph data, namely, graphons, as we explain in Subsection~\ref{sec:GCNNon}.
 
\begin{definition}
\label{def:GCNN}
Let $L\in\mathbb{N}$ denote the number of layers.
Let $F_l \in \mathbb{N}$ denote the number of features in the layer $l \in \{1,\ldots,L\}$. 
For each $l \in \{1,\ldots,L\}$, let $H_l:=(h_l^{jk})_{j \leq F_l,k \leq F_{l-1}} \in (C(\mathbb{R}))^{F_l \times F_{l-1}}$ be an array of  filter functions
and $\mathbf{M}_l :=(m_l^{jk})_{j \leq F_l,k \leq F_{l-1}} \in \mathbb{R}^{F_l \times F_{l-1}} $ a matrix. Let $\rho: \mathbb{R} \rightarrow \mathbb{R}$ be a continuous function that we call the \emph{activation function}. Then, the corresponding \emph{spectral convolutional neural network (SCNN) } $\phi$ is defined to be
\begin{equation}
    \phi := (H, \mathbf{M}, \rho) := ((H_l)_{l=1}^L, (\mathbf{M}_l)_{l=1}^L,  \rho).
\end{equation}
\end{definition}

In comparison to the single filter case in Section~\ref{subsection:GraphFilters}, a graph signal is not  restricted to be an element in $\mathbb{C}^n$. We consider  \emph{feature maps} given by a tuple of   graph signals, i.e., by a $\mathbf{x} := (\mathbf{x}^k)_{k=1}^{F_l} \in \mathbb{C}^{n \times {F_l}}$, where $F_l\in\mathbb{N}$ is the feature dimension.

\begin{definition}
\label{def:GCNNRealization}
Let $G$ be a graph with GSO $\DD$ and $\phi$ be an SCNN, as defined in Definition~\ref{def:GCNN}. For each $l \in \{1,\ldots,L\}$, we define
$\phi_\DD^l$ as the mapping that maps input feature maps to the feature maps in the $l$-th layer, i.e.,
\begin{equation}
    \phi_\DD^l: \mathbb{C}^{n \times {F_0}} \rightarrow \mathbb{C}^{n \times {F_l}} , \,  
    \mathbf{x}  \mapsto \mathbf{x}_l =  (\mathbf{x}_l^j)_{j=1}^{F_l}
\end{equation}
with
\begin{equation}
\label{eq:OutputGCNN}
    \mathbf{x}_l^j = \rho(\sum_{k=1}^{F_{l-1}} m_l^ {jk}h_l^{jk}(\DD)(\mathbf{x}_{l-1}^k)), 
\end{equation}
for $j \in \{1,\ldots,F_l\}$ and $\mathbf{x}_0 = \mathbf{x}$. 
We write $ \phi_\DD :=  \phi_\DD^L$ and call it the realization of $\phi$ on the graph $G$ with GSO $\DD$. Furthermore, we define $\tilde{\phi}_\DD^l$ as the mapping that maps input features to the feature maps in the $l$-th layer before applying the activation function in the $l$-th layer, i.e.,
\[
\tilde{\phi}_\DD^l: \mathbb{C}^{n \times {F_0}} \rightarrow \mathbb{C}^{n \times {F_l}} , \,  
    \mathbf{x}  \mapsto \tilde{\mathbf{x}}_l =  (\tilde{\mathbf{x}}_l^j)_{j=1}^{F_l}
\]
with
\begin{equation*}
    \tilde{\mathbf{x}}_l^j = \sum_{k=1}^{F_{l-1}} m_l^{jk}h_l^{jk}(\DD)(\phi_\DD^{l-1}(\mathbf{x})^k), 
\end{equation*}
for $j \in \{1,\ldots,F_l\}$. 
\end{definition}

For brevity, we typically do not distinguish between an SCNN and its realization on a graph.

\section{Graphon Analysis} 
\label{sec:ClassicalGraphons}
In this section, we recall the notion of a \emph{graphon},  which we describe informally as limit objects of sequences of graphs in Subsection~\ref{subsubsec:transf}.
The definition of convergence in the graphon setting is based on homomorphism numbers, which we recall in Subsection~\ref{subsec:homnumbers}. We give the definition of graphs converging to graphons in Subsection~\ref{subsec:graphons}.   In Subsection~\ref{subsec:GRSO} we discuss shift operators acting on the space of graphons.
\subsection{Homomorphism Numbers}
\label{subsec:homnumbers}

Lov\'{a}sz and Szegedy introduce in~\cite{10.1016/j.jctb.2006.05.002} the following definition.
\begin{definition}
\label{def:graphHomDensity}
Let $F=(V(F), E(F))$ be a simple graph and $G$ a graph with GSO $\DD$ and node set $V(G)$. Let $\bs{A}=n\DD$ be the GWM of G.
\begin{itemize}
    \item[$(1)$] Given a map  $\psi:V(F) \rightarrow V(G)$, we define  the homomorphism number
 \begin{equation*}
     hom_{\psi} (F, G) = \prod_{(u, v) \in E(F) } \bs{A}_{(\psi(u), \psi(v))}.
 \end{equation*}
\item[$(2)$] We define the homomorphism number
 \begin{equation}
 \label{weightedGraphHomDensity} 
          hom (F, G) = \sum_{\psi:V(F) \rightarrow V(G) } hom_{\psi} (F, G),
 \end{equation}
 where the sum goes over all maps $\psi:V(F) \rightarrow V(G)$.
\item[$(3)$]The \emph{homomorphism density} with respect to the graph $G$ is a function $t ( \cdot, G)$ that assigns to each simple graph $F$  the following value
\begin{equation}
\label{eq:homDensity}
t ( F , G ) := \frac{| h o m ( F , G ) |}{| V ( G ) |^{ | V ( F ) |}}.
\end{equation}
\end{itemize}
\end{definition} 

Let us give an interpretation for the homomorphism numbers.
Suppose that $G$ is a simple graph with adjacency matrix $\bs{W} \in \{0,1\}^{n\times n}$ and GWM $\bs{A} = \bs{W}$.  We recall that a homomorphism between two graphs is an adjacency-preserving mapping between their nodes. 
 Then, $hom_{\psi}(F, G)$ is one if $\psi$ is a homomorphism and zero otherwise. $hom(F,G)$ counts the number of homomorphisms of $F$ into $G$.  Since $|V(G)|^{|V(F)|} $ is the total number of maps between the nodes of $F$ and $G$, the homomorphism density $t(F,G)$ gives the probability that a random map from $F$ to $G$ is a homomorphism. We interpret $F$ as a simple substructure, called a \emph{motif}, so that $hom(F,G)$ counts how many times the motif $F$ appears in  $G$. If $G$ is a weighted graph  with weights  in $[0,1]$, interpreted as the probability of edges to exist,  then $hom(F, G)$ is the probability of the motif $F$ to appear in $G$. For further interpretations, see~\cite[Chapter 5.3]{lovasz}.

\subsection{Graphons}
\label{subsec:graphons}
The seminal work~\cite{10.1016/j.jctb.2006.05.002} studied sequences of graphs $G_n$  such that for every motif $F$  the parameter $t(F,G_n)$ converges.  Graphons were shown to be the natural limit object of such sequences.
\begin{definition}
A \emph{graphon} is a bounded, symmetric and measurable function $ W: [0,1] \times [0,1] \rightarrow  \mathbb{R} $  \footnote{Graphons can be defined over any atom-free standard probability space (see Subsection
\ref{subsec:comparison}).}. We denote the space of all graphons by $\mathcal{W}$. 
Let $\Gamma > 0$. The subspace $\mathcal{W}_\Gamma \subset \mathcal{W}$ is defined as the set of all graphons $W \in  \mathcal{W}$ such that $|W(u, v)| \leq \Gamma$  for every $(u,v) \in {[0,1] \times [0,1]}$.
\end{definition}

Graphons are interpreted as the spaces from which graphs are sampled, and to which graphs are embedded in order to compare them. Since in applications we only observe the graphs, we call the graphon corresponding to a graph its \emph{latent space}.
We  do not distinguish between functions that are almost everywhere equal. 
The way graphons are seen as limit objects of graphs is via the homomorphism density of graphons.
\begin{definition}
\label{def:homdensityGraphon}
For a graphon $W$ and a simple graph $F=(V,E)$, the homomorphism density of  $F$ in $W$ is defined as 
\begin{equation}
\label{eq:homdensityGraphon}
    t(F,W):=  {\int_{[0,1]^{n}}} \prod_{(i,j) \in E} W(u_i, u_j) \prod_{i \in V} du_i,
\end{equation}
where $n$ is the number of nodes in $F$.
\end{definition}

Since the integral can be interpreted as a continuous analogue to the sum,~(\ref{eq:homdensityGraphon}) is a continuous analogue to the  homomorphism density for graphs with GSO, and interpreted as the density of the motif $F$ in $W$. We next  recall a notion of convergence of sequences of graphs~\cite[Section 2.2]{10.1016/j.jctb.2006.05.002}.
\begin{definition}
 \label{def:graphConvToGraphon}
Let $(G_n)_n$ be a sequence of graphs with GSOs and GWMs \newline $(\bs{A}_n)_n$.
 \begin{itemize}
     \item[$(1)$] 
 We say that $(G_n)_n$ is a bounded sequence if there exists a constant $C > 0$ such that  $\|\bs{A}_n \| \leq C$ for every $n \in \mathbb{N}$.
    \item[$(2)$] 
We say that $(G_n)_n$ is \emph{convergent in homomorphism density},
if for all simple graphs $F$ the sequence $(t(F, G_n))_n$ is convergent. 
 \end{itemize} 
\end{definition}

The following theorem, taken from~\cite[Theorem 3.8 and Corollary 3.9]{Borgs2007graph},  shows that graphons are  natural limit objects of convergent sequences of graphs.

\begin{theorem}[{{\cite[Theorem 3.8 and Corollary 3.9]{Borgs2007graph}}}]
\label{thm:existenceLimitObject}
For every bounded sequence of graphs $(G_n)_n$  with GSO that converges in homomorphism density, there exists a graphon $W \in \mathcal{W}$ such that for every simple graph $F$
$$
t(F, G_n) \xrightarrow{n \to \infty} t(F, W).
$$
\end{theorem}

\begin{remark}
    If $(G_n)_n$ is convergent according to Definition~\ref{def:graphConvToGraphon}, then given the graphon $W$ from Theorem~\ref{thm:existenceLimitObject}, we say that $(G_n)_n$ converges to $W$ in homomorphism density and write $G_n \xrightarrow[n \to \infty]{H}W$. {Conversely, given a graphon $W \in \mathcal{W}$, there exists a sequence of graphs $(G_n)_n$ such that $G_n \xrightarrow[n \to \infty]{H}W$. The sequence $(G_n)_n$ can be constructed by first sampling i.i.d. uniform random points and then using the values of $W$ to define the GWM (see \cite[Theorem 4.5]{Borgs2007graph}).}
\end{remark}

\begin{remark}
\label{remark:growingNodeSize}
In the following, we only consider a sequence of graphs $(G_n)_n$ with $v(G_n) \xrightarrow{n \to \infty} \infty$, where $v(G_n)$ denotes the number of nodes of $G_n$. For simplicity of notation, we assume that $v(G_n) = n$. 
\end{remark}

In the following, we describe a method for defining a graphon corresponding to the graph $G$ using the notion of induced graphons. 
{ For this, we define for every $n \in \mathbb{N}$ the standard partition $\mathcal{P}_n = \{P_1, \dots, P_n\}$ of $[0,1]$ by the sub-intervals $P_k = [k-1/n,k/n)$ for $k=1, \ldots, n$.} 

Let $\chi_A$ denote the characteristic function of a subset  {$A \subset [0,1]$}, i.e.,
$$
\chi_A(u)=
\begin{cases}
1, & \text{ for } u \in A\ \\
0 & \text{ else. }
\end{cases}
$$
We use characteristic functions to define the induction of a graph to a graphon.

\begin{definition}
\label{def:inducedGraphon}
Let $G$ be a graph with GSO $\DD$ and GWM $\bs{A}$ and $n$ nodes.
The graphon $W_{\AA}$ \emph{induced} by $G$ is defined as 
\begin{equation}
\label{def:inducedgraphon2}
    W_{\AA}(u,v) := \sum_{i,j \leq n}\AA_{(i,j)} \chi_{P_i}(u)\chi_{P_j}(v).\footnote{
{
We note that we could use any partition $\tilde{\mathcal{P}} = \{ \tilde{P}_1, \ldots, \tilde{P}_n \}$ such that $\mu(\tilde{P}_i) = 1/n$ for every $i=1, \ldots, n$ and $[0,1]$ is the disjoint union of all element in $\tilde{\mathcal{P}}$.} 
}
\end{equation}
\end{definition}

{With Definition \ref{def:inducedGraphon}, we clearly have
\begin{equation}
\label{eq:HomDensityEqualForInducedGraphon}
    t(F, G) = t(F, W_{\AA})
\end{equation}
for every graph $G$ with GWM $\AA$ and every simple graph $F$. }We recall that GSOs $\DD$ and filters $h(\DD)$, which are GSOs themselves, are related to GWMs by~(\ref{eq:GWM}). Hence, we use the notation $W_{nh(\DD)}$ to describe the graphon induced by the GWM corresponding to the GSO $h(\DD)$.

 Next we discuss a setting for performing signal processing with graphons, using $L^2$-norms. 
  {The} space of \emph{graphon signals} is defined as  {$L^2[0,1]$}.
  In the context of signal processing, we define convergence of graphons by interpreting them as kernel operators.
  {For this, we recall that in Subsection \ref{subsec:contributions} we defined the  graphon shift operator (GRSO) corresponding to the graphon $W \in \mathcal{W}$ as
the operator
\begin{equation} 
\label{eq:integralLaplacian}
T_{W} \psi (v) :=  \int_0^1 W(v, u) \psi(u) d\mu(u)
\end{equation}
for signals $\psi \in L^2[0,1]$. The operator norm of $T_W$ is closely related to convergence in homomorphism density. }
  To understand this relation, we define the \emph{relabeling} of a graph by a permutation $\pi$.
\begin{definition}
Let $G$ be a graph with GSO $\DD$  and edge set $E$. 
We define $\Pi_G$ as the set of all permutations of the node set of $G$.
Let $\pi \in \Pi_G$.
The graph $\pi(G)$ is defined as the graph 
$\pi(G) = ( \pi(V), \pi(E), \pi(\DD) )$  with GSO where $(i,j) \in   \pi(E)$ if there exists an edge $(k,l) \in E$ such that $(\pi(k), \pi(l)) = (i,j)$. The GSO is defined by $\pi(\DD) = (\delta_{\pi(i), \pi(j)})_{i,j}$, where $\DD = (\delta_{i,j})_{i,j}$. 
We call $\pi(G)$ the \emph{relabeled graph}.
\end{definition}

 {The following result relates convergence in homomorphism density to the operator norm. }
\begin{theorem}
\label{thm:cutnormconv2}
Let $(G_n)_{n}$ a sequence of graphs with GSOs $(\DD_n)_n$ and $(\AA_n)_n$ be the corresponding sequence of  GWMs. Let $W \in \mathcal{W}$ satisfy $G_n \xrightarrow[n\to\infty]{H} W$. Let $(W_{\AA_n})_n$  be the sequence of graphons induced by $(G_n)_{n}$. Then, there exists a sequence of permutations $(\pi_n)_{n}$ such that 
{$
\| {T_{W_{\pi_n(\AA_n)}}} - T_W\|_{L^2[0,1] \to L^2[0,1]} \xrightarrow{n \to \infty} 0
$}.
\end{theorem}

{Theorem \ref{thm:cutnormconv2} follows from a well-known result in  graphon theory (see \ref{Appendix:Graphons}). In fact, the converse direction in Theorem \ref{thm:cutnormconv2} is also true, showing the equivalence between convergence in homomorphism density and convergence of the induced GRSOs in the $L^2$-induced operator norm up to relabelling of the nodes. }

We now follow {a similar} route as Ruiz et al. in~\cite{RuizI} to define convergence of graph signals to graphon signals.   Given a sequence of graph signals $(\mathbf{x}_n)_n$ on graphs $(G_n)_n$ with a growing number of nodes, we can interpret graphon signals as the limit of the graph signal sequence. We recall the definition of induced graphon signals from graph signals~\cite{RuizI}. 

\begin{definition}
\label{definition:inducedsignal}
Let $\mathbf{x} \in \mathbb{C}^n$ be a signal on a graph $G$.   The graphon signal $\psi_\mathbf{x} \in L^2(\mathcal{J})$ \emph{induced} by $\mathbf{x}$ is defined as
\begin{equation}
\label{def:inducedsignal}
    \psi_\mathbf{x}(u) := \sum_{j \leq n} x_j \chi_{P_j}(u).
\end{equation}
\end{definition}

We then define the following set of admissible permutations.
\begin{definition}
Let $(G_n)_n$ be a sequence of graphs with GSOs and corresponding GWMs $(\AA_n)_n$. Let $W \in \mathcal{W}$ satisfy $G_n \xrightarrow[n\to\infty]{H} W$.  
The set of \emph{admissible permutations} is defined as 
$$
\mathbf{P} = \left\{ (\pi_n)_n  \in (\Pi_{G_n})_n \, | \,  {\| {T_{W_{\pi_n(\AA_n)}}} - T_W\|_{L^2[0,1] \to L^2[0,1]}} \xrightarrow{n  \to \infty} 0 \right\}.
$$ 
\end{definition}

Theorem~\ref{thm:cutnormconv2} ensures that the set of admissible permutations is not empty. 
We use the set of admissible permutations to define convergence of graph signals to graphon signals.
\begin{definition}
\label{def:convergenceOfSignal}
Let $(G_n)_n$ be a sequence of graphs with GSOs and $(\mathbf{x}_n)_n$  a sequence of graph signals. We say that $(\mathbf{x}_n)_n$ converges to a graphon signal  {$\psi \in L^2[0,1]$ } if there exists a graphon $W$ such that $G_n \xrightarrow[n \to \infty]{H} W$ and for some admissible permutation $ (\pi_n)_n \in \mathbf{P}$
\begin{equation}
\|\psi_{\pi_n(\mathbf{x}_n)} - \psi\|_{ {L^2[0,1]}} \xrightarrow{n\to\infty} 0.
\end{equation}
We write $\mathbf{x}_n \xrightarrow[n \to \infty]{I} \psi$.
\end{definition}

\subsection{Graphon Shift Operators and Graphon Filters} 
\label{subsec:GRSO} 
We define realizations of filters acting on graphon signals using GRSOs. Since $W$ is symmetric, $T_W$ is a self-adjoint Hilbert-Schmidt operator, and hence has a discrete eigendecomposition.
\begin{definition}
Let $h: \mathbb{R} \rightarrow \mathbb{R}$ be a filter. Let $W \in \mathcal{W}$,  and let  $T_W$ be the associated GRSO. 
Let $\{\lambda_j, \phi_j \}_j$ be the    eigendecomposition  {(with possibly repeating eigenvalues and orthonormal eigenvectors)} of $T_W$. The \emph{realization of the filter} $h(T_W)$ is defined as
$$
h(T_W) \psi = \sum_{j = 1}^\infty h(\lambda_j) \langle \psi, \phi_j\rangle \phi_j
$$
for  {$\psi \in L^2[0,1]$}. 
\end{definition}

Similarly to graph filters, we often call $h(T_W)$ simply a filter.
\begin{example}
If $h(\lambda) = \sum_{j=0}^n p_j \lambda^j / \sum_{i=0}^m q_i \lambda^i$ is a rational function, the filter $h(T_W)$ is given by
$$
h(T_W)  = \left( \sum_{j=0}^n p_j T_W^j \right) \left( \sum_{i=0}^m q_i T_W^i  \right)^{-1}.
$$
\end{example}

\subsection{Graphon Convolutional Neural Networks (GCNNon)}
\label{sec:GCNNon}
In this subsection, we define graphon convolutional neural networks (GCNNons).
The GCNNon architecture can be defined equivalently to the one of GCNNs, given in Subsection~\ref{subsec:GCNN}.
GCNNons are defined equivalently to GCNNs, given in Subsection~\ref{subsec:GCNN}, by applying an SCNN on a GRSO and a graphon signal.
The construction is akin to  WNNs in~\cite{ruiz2021graph}   and continuous GCNs in~\cite{keriven2020convergence}. Nevertheless, we have fewer constraints on the filters  and the graphon.

Let $\phi$ be SCNN. Let $W \in \mathcal{W}$ be a graphon and $T_W$ the GRSO from~(\ref{eq:integralLaplacian}). Similarly to GCNNs, we define the mapping $\phi_W^l$ by only replacing the discrete GSO by the GRSO $T_W$ and propagating square-integrable functions instead of discrete graph signals in Definition~\ref{def:GCNNRealization}. This means that the GCNNon  is based on  vector valued signals   {$ \psi  := (\psi^g)_{g  = 1}^{ F_l} \in (L^2[0,1])^{F_l}$}  that we call \emph{graphon features maps}, as inputs and outputs.

\section{The Transferability Framework}
\label{sec:transframework}
Induced signals and induced graphons provide a framework to compare the repercussion of filters acting on graphs that approximate the same graphon. 
Given two graphs $G_1, G_2$ with GSOs $\DD_1 \in \mathbb{R}^{n_1\times n_1}, \DD_2 \in \mathbb{R}^{n_2 \times n_2}$, we compare the repercussions of the filters $h(\DD_1)$ and $h(\DD_2)$ as follows.
We map the filters into the latent space by defining the induced graphons $W_{n_1h(\DD_1)}$ and $W_{n_2h(\DD_2)}$. Then, the repercussion error of the filters $h(\DD_1)$ and $h(\DD_2)$ in the latent space is defined by 
$${
\|T_{W_{n_1h(\DD_1)}} - T_{W_{n_2h(\DD_2)}}\|_{L^2[0,1] \rightarrow L^2[0,1]}.}
$$
 
For graph signals $\mathbf{x}_1 \in \mathbb{C}^{n_1}$ and $\mathbf{x}_2 \in \mathbb{C}^{n_2}$ with 
$
\|\psi_{\mathbf{x}_1} - \psi_{\mathbf{x}_2}{\|_{L^2[0,1]}}
$ sufficiently small,
the repercussion error of the filter outputs of $h(\DD_1)$ and $h(\DD_2)$ is defined by 
$$
\|\psi_{h(\DD_1) \mathbf{x}_1} - \psi_{h(\DD_2) \mathbf{x}_2} {\|_{L^2[0,1]}}.
$$
In this section, we develop tools for analyzing the above repercussion errors and extend the analysis to accommodate unbounded operators on graphon signals.

\subsection{The Spaces of Induced Graphons and Induced Signals}
We study in this subsection the relation between the space of graph signals $\mathbb{C}^n$ and its image under the induction operation of Definition~\ref{definition:inducedsignal}. The image space of the induction operator is a space of step functions.
\begin{definition}
{For $n \in \mathbb{N}$, we}  define the space of step functions   {$[0,1]/\mathcal{P}_n \subset L^2[0,1]$}  as  ${[0,1]/\mathcal{P}_n} = {\rm span}\{\chi_{P_1}, \ldots, \chi_{P_n}\}$. 
\end{definition}

The following lemma can be easily verified.
\begin{lemma}
\label{remark:IsomIso}
  { For $n \in \mathbb{N}$, the correspondence $\mathbf{x} \mapsto \psi_\mathbf{x}$, defined in  Definition~\ref{definition:inducedsignal}, between the space $\mathbb{C}^n$ and $[0,1]/\mathcal{P}_n$  is an isomorphism. Further, it holds
  \[
  \frac{1}{\sqrt{n}} \|\mathbf{x}\|_{\mathbb{C}^n} = \|\psi_\mathbf{x}\|_{L^2[0,1]} \, \text{ and } \, \frac{1}{n} \langle \mathbf{x}, \mathbf{y} \rangle_{\mathbb{C}^n} =  \langle \psi_\mathbf{x}, \psi_\mathbf{y} \rangle_{L^2[0,1]} \]
  for all $\mathbf{x},\mathbf{y} \in \mathbb{C}^n$.
  } 
\end{lemma}

\begin{lemma}
\label{lemma:eigenvectorInducedGRSO}
Let $G$ be a graph with GSO $\DD$ and associated GWM $\AA$. Denote the eigendecomposition of $\DD$ by $\{\lambda^i, \mathbf{x}^i \}_{i=1}^n$. Then,
$T_{W_{\AA}}$ admits the eigendecomposition  {$\{\lambda^i, \sqrt{n} \psi_{\mathbf{x}^i} \}_{i=1}^n \cup \{ 0, \psi_j\}_{j=n+1}^\infty$}, where the eigenvectors are an orthonormal basis for {$L^2[0,1]$}.
\end{lemma}
\begin{proof}
 Let $\{\lambda^i, \mathbf{x}^i\}_{i=1}^n$ denote the eigendecomposition of the GSO $\DD$ with $|\lambda_1| \leq |\lambda_2| \leq \ldots \leq |\lambda_n|$.  Let $N \in \{1, \ldots, n\}$. Then, for almost every $u \in   {[0,1]}$,
$$
\begin{aligned} 
        T_{W_{\AA}}( {\sqrt{n}} \psi_{\mathbf{x}^N})(u)& = 
        {\int_0^1}
        ({\sqrt{n}}\sum_{k,l \leq n} \chi_{P_k}(u) \chi_{P_l}(v) \AA_{(k,l)}) (\sum_{j \leq n} \chi_{P_j}(v) x^N_j) dv \\
        & = \sum_{k \leq n} \chi_{P_k}(u) {\int_0^1}({\sqrt{n}}\sum_{l, j \leq n} \chi_{P_j}(v) \chi_{P_l}(v) \AA_{(k,l)}   x^N_j)  dv \\
        & =  \sum_{k \leq n} \chi_{P_k}(u) 
        {\int_0^1}({\sqrt{n}}\sum_{j \leq n} \chi_{P_j}(v)   \AA_{(k,j)} x^N_j)  dv \\
        & =  \sum_{k \leq n}  \chi_{P_k}(u) \frac{1}{n} {\sqrt{n}} \sum_{j \leq n}   \AA_{(k,j)} x^N_j  
        \\ & =  \sum_{k \leq n}  \chi_{P_k}(u)  {\sqrt{n}}  \sum_{j \leq n}   \DD_{(k,j)} x^N_j \\
        & =  {\sqrt{n}}\sum_{k \leq n}  \chi_{P_k}(u)   \lambda^N x^N_k 
        =   \lambda^N {\sqrt{n}} \psi_{\mathbf{x}^N}(u).
\end{aligned}
$$

Hence, for every eigenvalue-eigenvector pair $(\lambda^N,\mathbf{x}^N)$ of $\DD$, $(\lambda^N, {\sqrt{n}} \psi_{\mathbf{x}^N})$ is one of  $T_{W_{\AA}}$. Since the image of $T_{W_{\AA}}$ is ${[0,1]}/\mathcal{P}_n$ and the induced signals $\{ {\sqrt{n}}\psi_{\mathbf{x}^1}, \ldots, {\sqrt{n}}\psi_{\mathbf{x}^{{n}}} \}$ form an orthonormal basis for ${[0,1]}/\mathcal{P}_n$,  there can not exist any other non-zero eigenvalue. 
\end{proof}

The following lemma shows that given a graph $G$ with GSO $\DD$ and corresponding GWM $\AA$, we can either first induce it to the graphon space $W_{\AA}$, then build the GRSO $T_{W_{\AA}}$ and the filter $h(T_{W_{\AA}})$, or first apply the discrete filter on the GSO $h(\DD)$ and, only then, induce it to the graphon space. 
\begin{proposition}
\label{prop:graphdomain1}
    Let $G$ be a graph with GSO $\DD$ and corresponding GWM $\AA$. Let $h: \mathbb{R} \rightarrow \mathbb{R}$ be continuous with $h(0) = 0$. Then, 
$
h(T_{\AA}) = T_{W_{n h(\DD)}}
$.
\end{proposition}
\begin{proof}
  Let $\{\lambda^i, \mathbf{x}^i \}_{i=1}^n$ denote the eigendecomposition of $\DD$. Then, $h(\DD)$ admits the eigendecomposition $\{h(\lambda^i), \mathbf{x}^i \}_{i=1}^n$. By Lemma~\ref{lemma:eigenvectorInducedGRSO}, the eigendecomposition of  $T_{W_{nh(\DD)}}$ is $\{ h(\lambda^i), {\sqrt{n}} \psi_{\mathbf{x}^i} \}_{i=1}^n $ with infinitely many eigenvectors corresponding to the kernel $\{ 0, \psi_j \}_{j=n+1}^{\infty}$.

We compare this to the eigendecomposition of $h(T_{W_{\AA}})$. By Lemma~\ref{lemma:eigenvectorInducedGRSO}, $T_{W_{\AA}}$ has eigendecomposition $ \{ \lambda^i, {\sqrt{n}}\psi_{\mathbf{x}^i} \}_{i=1}^n $ with infinitely many eigenvectors corresponding to the kernel   $\{ 0, \psi_j \}_{j=n+1}^{\infty}$. Then, the eigendecomposition of $h(T_{W_{\AA}})$ is given by $\{ h( \lambda^i),{\sqrt{n}} \psi_{\mathbf{x}^i} \}_{i=1}^n  \cup \{ h(0), \psi_j \}_{j=n+1}^{\infty}$. The condition $h(0) = 0 $, leads to $h(T_{W_{\AA}}) = T_{W_{nh(\DD)}}$.
\end{proof}

 The condition $h(0) = 0$ in Proposition~\ref{prop:graphdomain1} can be removed if we restrict the input signals to induced graphon signals $\psi_\mathbf{x}$, instead of studying $T_{W_{nh(\DD)}}$ and $h(T_{W_{\AA}})$ on the whole space $L^2(\mathcal{J})$.

\begin{proposition}
\label{prop:graphdomain2}
Let $G$ be a graph  with GSO $\DD$ and corresponding GWM $\AA$. Let $h: \mathbb{R} \rightarrow \mathbb{R}$ be a continuous function. Then, for every signal $\mathbf{x} \in \mathbb{C}^n$
\begin{equation}
    \psi_{h(\DD)\mathbf{x}} = h(T_{W_{\AA}}) \psi_{\mathbf{x}}.
\end{equation}
\end{proposition}
\begin{proof}
Similarly to the calculations and notations from the proof of  Proposition~\ref{prop:graphdomain1}, we can write 
$$
h(\DD) \mathbf{x} = \sum_{i=1}^n h(\lambda^i) \langle  \mathbf{x},\mathbf{x}^i \rangle_{\mathbb{C}^n} \mathbf{x}^i
$$ 
and
$$
h(T_{W_{\AA}}) \psi_\mathbf{x} = \sum_{i=1}^n  h(\lambda^i) \langle  \psi_\mathbf{x}, {\sqrt{n}}\psi_{\mathbf{x}^i} \rangle_{L^2(\J)} {\sqrt{n}}\psi_{\mathbf{x}^i}
$$
for $\mathbf{x} \in \mathbb{C}^n$.
By Definition~\ref{definition:inducedsignal} of induced signals, we have
$$
\begin{aligned}
  \psi_{h(\DD)\mathbf{x}}(\cdot) & =  \sum_{j=1}^n \chi_{I_j}(\cdot) \sum_{i=1}^n h(\lambda^i) \langle  \mathbf{x},\mathbf{x}^i  \rangle_{\mathbb{C}^n} x^i_j 
   =   \sum_{i=1}^n h(\lambda_i) \langle \mathbf{x},\mathbf{x}^i  \rangle_{\mathbb{C}^n}\sum_{j=1}^n \chi_{I_j}(\cdot) x^i_j \\
  & = \sum_{i=1}^n h(\lambda_i) \langle  \mathbf{x},\mathbf{x}^i  \rangle_{\mathbb{C}^n} \psi_{\mathbf{x}^i} =  \sum_{i=1}^n h(\lambda_i) \langle   \psi_\mathbf{x}, {\sqrt{n}}\psi_{\mathbf{x}^i}  \rangle_{L^2(\J)} {\sqrt{n}}\psi_{\mathbf{x}^i} \\
  & = h(T_{W_{\AA}}) \psi_\mathbf{x}.
  \end{aligned}
$$
The second to last equation holds  {by Remark \ref{remark:IsomIso}}. 
\end{proof}

\subsection{General Unbounded Shift Operator}
\label{subsec:unbounded}
 Graphon shift operators cannot represent all important limit objects of GSOs. An important example are filters in  Euclidean domains, which are the backbone of standard convolutional neural networks.
 \begin{example}
 \label{ex:GRSO1}
Let $D$ denote the space of $1$-periodic functions on $\mathbb{R}$ with absolutely continuous first derivative, 
equipped with the $L^2[0,1]$ inner product. Let $\Delta$ be the Laplace operator,
$$
\Delta  : D \rightarrow  D , \quad
f \mapsto  -f''.
$$
 Let $h$ be a function and $(h_n)_n = (h(4\pi^2n^2))_n$. Define the Euclidean circular convolution operator\footnote{A circular convolution in $L^2[0,1]$ between $f$ and $q$ is defined by  $f*q(x) = \int_0^{1} f(y)q(x-y)dy$, where values of  $q$ outside $[0,1]$ are taken from the periodic extension of $q$.} $C_h$  by $$C_h f =
 f*(\sum h_n e^{2\pi i n (\cdot)})$$ for any input signal $f$. Operators of this kind exhaust the space of all circular convolution operators with even kernels.
 By the convolution theorem, this can be written in the frequency domain as
 \[\mathcal{F} C_h \mathcal{F}^* (\hat{f}_n)_n = (h_n \hat{f}_n)_n.\]
 Applying the inverse Fourier transform leads to
 $$
 C_h f = \sum_{n} \hat{f}_n h_n e^{2\pi i n (\cdot)} = h(\Delta) f.
 $$
 This means that Euclidean circular convolution operators  can be written as functional calculus filters of the Laplace operator.
 
 \end{example}


We consider a discretization of $[0,1]$ by a grid.  The grid is a graph, defined by connecting each grid point to its neighbors by edges.  This means that $[0,1]$, with a grid, are a special case of a graphon latent space and a graph. 
In this subsection, we show algorithmic alignment of GCNNons with CNNs on Euclidean data, where filters are defined by functional calculus of the   {Laplace operator} as described in Example \ref{ex:GRSO1}.
The Laplace operator is not bounded, so there exists no graphon $W$ with $\Delta = T_W$.
To be able to treat discrete grid convolutions as graph convolutions approximating $\Delta$, our analysis should allow   \emph{unbounded GRSOs}. 
 We first define the Paley-Wiener spaces of self-adjoint operators.

\begin{definition}
\label{def:PWspace}
Let  $\mathcal{L}$ be {a} self-adjoint operator with spectrum consisting only of eigenvalues. Denote the eigenvalues, eigenspaces, and projections upon the eigenspaces of $\mathcal{L}$ by $\{\lambda_j, W_j, P_j \}_{j\in \mathbb{N}}$. For each $\lambda > 0$, we define the  $\lambda$'th \emph{Paley-Wiener} space of $\mathcal{L}$ as
$$
PW_\mathcal{L}(\lambda) = \bigoplus_{j \in \mathbb{N}} \{ W_j \, | \, |\lambda_j| \leq \lambda \}.
$$
The  \emph{spectral projection} $P_\mathcal{L}(\lambda)$ upon $PW_\mathcal{L}(\lambda)$ is defined by 
$$
P_\mathcal{L}(\lambda) = \sum_{|\lambda_j| \leq \lambda} P_j.
$$
\end{definition}

\begin{definition}
\label{def:unboundedGSRO}
Let $\mathcal{L}$ be {a} self-adjoint operator defined on a dense subset of {$L^2[0,1]$}, with spectrum consisting only of eigenvalues. We say that $\mathcal{L}$ is an \emph{unbounded graphon shift operator (unbounded GRSO)} if it is self-adjoint and for every $\lambda > 0$, there exists a graphon $W^{\lambda} \in \mathcal{W}$ such that
\begin{equation}
    {\mathcal{L} P_\mathcal{L}(\lambda)  = T_{W^\lambda}.}
\end{equation}
\end{definition}


The following proposition can be easily verified.

\begin{proposition}
\label{prop:UnboundedGRSOEq}
A self-adjoint operator $\mathcal{L}$ with spectrum $\sigma(\mathcal{L}) = \{\lambda_1, \lambda_2, \ldots\}$ consisting only of eigenvalues is an unbounded GRSO if and only if the eigenvalues  in each compact interval are square summable, i.e., for every $a<b \in \mathbb{R}$
$$
 \sum_{ \lambda \in \sigma(\mathcal{L}  )\cap [a, b]} \lambda^2 < \infty.
$$
\end{proposition}

Next, we define convergence of GSOs to unbounded graphon shift operators. 

\begin{definition}
\label{def:unbddconv}
Let $(G_n)_n$ be a sequence of graphs with GSOs $(\DD_n)_n$ and $(\AA_n)_n$ be the associated sequence of GWMs. We say that $(G_n)_n$ converges to the unbounded GRSO $\mathcal{L}$
if there exists a sequence of permutations $(\pi_n)_n$ such that for every $\lambda \in \mathbb{R}$ 
\begin{equation}
  \| P_\mathcal{L}(\lambda) T_{W_{\pi_n({\AA_n})}} P_\mathcal{L}(\lambda) -  \mathcal{L} P_\mathcal{L}(\lambda)   {\|_{L^2[0,1] \rightarrow L^2[0,1]}} \xrightarrow{n \to \infty} 0.
\end{equation}
We write $G_n \xrightarrow[n \to \infty]{U} \mathcal{L}$.
\end{definition}

In the following example, we show that the classical Laplace operator is an unbounded GRSO.
\begin{example}
\label{ex:LaplaceOperator}
Let $\Delta$ be the Laplace operator from Example \ref{ex:GRSO1}.
 The spectrum of $\Delta$  consists only of eigenvalues, given by
$$
\sigma(\Delta) = \{ 0, 4\pi^2, 4\pi^2 2^2, \ldots , 4 \pi^2 k^2, \ldots \}.
$$
The corresponding eigenvectors $e^{2 i \pi k x}$ form the classical Fourier basis.
For every $a<b \in \mathbb{R}$, the intersection $\sigma(\Delta) \cap [a, b]$ is a finite set. Hence, by Proposition \ref{prop:UnboundedGRSOEq}, the Laplace operator $\Delta$ is an unbounded GRSO.
\end{example}

There exists a sequence of graphs approximating the Laplace operator in the sense of Definition~\ref{def:unbddconv}. 
 The construction of the approximating sequence of graphs can be taken from the proof of Proposition \ref{prop:LaplaceOperatorAppro}, showing a more general result.

\begin{proposition}
\label{prop:LaplaceOperatorAppro}
Let $\mathcal{L}$ be an unbounded GRSO.
Then, there exists a sequence of {weighted} graphs $(G_n)_n$ with GSOs $(\DD_n)_n$ such that
$G_n \xrightarrow[n \to \infty]{U} \mathcal{L}$.
\end{proposition}
The proof of Proposition \ref{prop:LaplaceOperatorAppro} is given in   \ref{proof:LaplaceOperatorAppro}.

\section{Main Results on Transferability}
\label{sec:main}
 First, we give results on asymptotic transferability of continuous {graph} filters. Secondly, by restricting the analysis to filters with higher regularity, we achieve linear approximation speed. Then, we prove transferability results for graphs that approximate unbounded graphon shift operators. Last, we present transferability results for end-to-end graph convolutional neural networks.

\subsection{Asymptotic Transferability of Graph Filters}
\label{subsec:MainResult1}
Given a sequence of graphs $(G_n)_n$ with GSOs $(\DD_n)_n$ and a graphon $W$ such that $G_n \xrightarrow[n \to \infty]{H} W$, we show that the associated sequence of graph filters $(h(\DD_n))_n$ is convergent in operator norm. The limit object  is the filtered induced GRSO $h(T_W)$.
\begin{lemma}[Convergence of graph filters]
\label{claim:ConvOfGrFltBdd}
Let $(G_n)_n$ be a sequence of graphs with GSOs $(\DD_n)_n$ and its associated sequence of GWMs $(\AA_n)_n$. Let $W \in \mathcal{W}$ satisfy 
$G_n \xrightarrow[n \to \infty]{H} W$.
Let $h$ be a filter.
Then, there exists a sequence of permutations $(\pi_n)_n$ such that 
\begin{equation*}
    \|h(T_{W_{\pi_n({\AA_n})}}) -  h(T_W) {\|_{L^2[0,1] \to L^2[0,1]}} \xrightarrow{n \to \infty} 0.
\end{equation*}
{Consequently}, consider a sequence of graph signals $(\mathbf{x}_n)_n$  on the relabeled sequence of graphs. Assume  that the induced signals, defined in Definition~\ref{definition:inducedsignal}, converge to some $\psi \in   {L^2[0,1]}$, i.e., 
$
\|\psi_{\mathbf{x}_n} - \psi  {\|_{L^2[0,1]}} \xrightarrow{n \to \infty} 0
$. 
Then,
\begin{equation}
   \| h(T_{W_{\pi_n({\AA_n})}})\psi_{\mathbf{x}_n} - h(T_W)\psi {\|_{L^2[0,1]}} \xrightarrow{n \to \infty} 0.
\end{equation}
\end{lemma}
The proof for Lemma \ref{claim:ConvOfGrFltBdd} is given in \ref{Appendix:Graphons}.
In practice, filters are applied on graph shift operators. Hence, we want to show that the outputs of graph filters applied on GSOs, associated to graphs that approximate the same graphon, are approximately the same. The following corollaries formulate the transferability error by first applying the filter on GSOs and then taking the induced graphons.

\begin{theorem}
\label{cor:DiscrFilterConv}
Let $(G_n)_n$ be a sequence of graphs with GSOs $(\DD_n)_n$ and its associated sequence of GWMs $(\AA_n)_n$. Let $W \in \mathcal{W}$ satisfy 
$G_n \xrightarrow[n \to \infty]{H} W$. Let $h$ be a filter with $h(0) = 0$. Then,
$$
\|T_{W_{nh(\pi_n({\DD_n}))}} - T_{W_{mh(\pi_m({\DD_m}))}} {\|_{L^2[0,1] \to L^2[0,1]}}  \xrightarrow{n,m \to \infty} 0. 
$$
\end{theorem}
\begin{proof}
By Proposition~\ref{prop:graphdomain1} and Lemma~\ref{claim:ConvOfGrFltBdd}, we have
$$
\|T_{W_{nh(\pi_n({\DD_n}))}} - h(T_W) {\|_{L^2[0,1] \to L^2[0,1]}} \xrightarrow{n \to \infty} 0. 
$$
The proof now follows from the  triangle inequality.
\end{proof}

{ In the following result we get   rid of the assumption of $h(0)=0$ from Theorem \ref{cor:DiscrFilterConv} by considering the convergence of induced graphon signals instead of the convergence of induced GRSOs.} 

\begin{theorem}
\label{cor:ConvGraphFilter}
Let $(G_n)_n$ be a sequence of graphs with GSOs $(\DD_n)_n$ and its associated sequence of GWMs $(\AA_n)_n$. Let $W \in \mathcal{W}$ satisfy 
$G_n \xrightarrow[n \to \infty]{H} W$. Let $h$ be a filter.
Consider a sequence of graph signals $(\mathbf{x}_n)_n$ on $(G_n)_n$ and a graphon signal $\psi \in  {L^2[0,1]}$ such that 
$\mathbf{x}_n \xrightarrow[n \to \infty]{I} \psi$.
Then,
\begin{equation*}
\| \psi_{h (\DD_n) \mathbf{x}_n} - \psi_{h (\DD_m) \mathbf{x}_m} \|_{L^2[0,1]} \xrightarrow{n,m \to \infty} 0.
\end{equation*}
\end{theorem}
\begin{proof}
By Proposition~\ref{prop:graphdomain2} and the Lemma~\ref{claim:ConvOfGrFltBdd}, we have 
$
h (\DD_n) \mathbf{x}_n \xrightarrow[n \to \infty]{I} h(T_W) \psi
$. 
The proof now follows from the  triangle inequality.
\end{proof}

\subsection{Non-Asymptotic Transferability of Graph Filters}

In the following, we derive filters $h: \mathbb{R} \rightarrow \mathbb{R}$ that allow improved  rates of convergence. For $U \subset \mathbb{R}$ and $p \in \mathbb{N}$, we write $f \in C^p(U)$ if $f$ is $p$ times continuously differentiable for all $x \in U$.  
 
 
{We begin by calculating the convergence rate for the transferability in the \emph{Schatten $p$-norm}. For $ 1 \leq p \leq \infty$, the Schatten $p$-norm $\|\cdot \|_{S_p}$ of a compact self-adjoint operator is
defined as the $l_p$-norm of the
sequence of its eigenvalues. }

\begin{theorem}
\label{thm:LinConvHSnorm}
Let  $h$ be Lipschitz continuous   with Lipschitz constant $\|h\|_{Lip}$ and $h(0) = 0$. Let $(G_n)_n$ be a sequence of graphs with GSOs $(\DD_n)_n$ and associated GWMs $(\AA_n)_n$ such that there exists a $W \in \mathcal{W}$ with  $G_n\xrightarrow[n \to \infty]{H} W$. There exists a sequence of permutations $(\pi_n)_n$ such that {for every $2 <  p < \infty$}
\[
\begin{aligned}
 & \| T_{W_{nh (\pi_n(\DD_n))}}  - T_{W_{mh (\pi_m(\DD_m))}} {\|_{S_p}} \\ & \leq 2\|h\|_{Lip} {K_p  \|  T_{W_{\pi_n(\DD_n)}}  - T_{W_{\pi_m(\DD_m)}} \|_{S_p}}  \xrightarrow{n,m \to \infty} 0,
\end{aligned}
\]
{where $K_p$ is a universal constant described in \cite[Theorem 1]{OLinSchattenNeumannClasses}.}
\end{theorem}

{
\begin{remark}
If we restrict ourselves in Theorem \ref{thm:LinConvHSnorm} to graphons $W \in \mathcal{W}_\Gamma$ for a $\Gamma > 0$, the convergence rate in Theorem  \ref{thm:LinConvHSnorm} holds for any filter $h$ that is Lipschitz continuous in a neighborhood of $[-\Gamma, \Gamma]$. Thus, the convergence rate in  Theorem  \ref{thm:LinConvHSnorm} holds also for polynomials.
\end{remark}
}

The proof for Theorem \ref{thm:LinConvHSnorm} is given in \ref{subsec:proofConvRateHSNorm}.
We give another formulation of transferability based on Fourier coefficients of the filter $h$.
\begin{theorem}[Linear Convergence]
\label{convRateC1filter}
Let $\Gamma \in \mathbb{N}$  and $h \in C^1([-\Gamma, \Gamma])$ such that $h'$ is Lipschitz and $h(0) = 0$. 
Suppose that $G_1$ and $G_2$ are graphs with GSOs $\DD_1 \in \mathbb{R}^{n_1 \times n_1}$ and $\DD_2 \in \mathbb{R}^{n_2 \times n_2}$. Let $\AA_1$ and $\AA_2$ be the associated GWMs and suppose that there exists a $\Gamma > 0 $ with $\|\AA_1\|_{\mathbb{C}^{n_1} \rightarrow \mathbb{C}^{n_1}}, \|\AA_2 \|_{\mathbb{C}^{n_2} \rightarrow \mathbb{C}^{n_2}} \leq \Gamma$.
{We have for all permutations $\pi_1$ and $\pi_2$ of the graph's nodes }
{\begin{equation}
\label{eq:LinConv}
\begin{aligned}
 \|T_{W_{n_1h(\pi_1(\DD_1))}}  - T_{W_{n_2h(\pi_2(\DD_2))}} &   \|_{L^2[0,1] \to   L^2[0,1]} \\ & \leq C \|T_{W_{\pi_1(\AA_1)}} - T_{W_{\pi_2(\AA_2)}}  \|_{L^2[0,1] \to L^2[0,1]},
\end{aligned}
\end{equation}}
where $C$ is given by (\ref{eq:FourierSeriesBoundDerivative}).
\end{theorem}


\subsection{Transferability Between Graphs and Unbounded GRSOs}
\label{main2}
In this subsection, we generalize the asymptotic transferability results of Subsection \ref{subsec:MainResult1} to unbounded GRSOs. The following lemma  is a direct consequence of Lemma \ref{lemma:contoffunccalc}.

\begin{lemma}
\label{thm:TransfUnbddL}
Let $h: \mathbb{R} \rightarrow \mathbb{R}$ be a  continuous function.  Let $(G_n)_n$ be a sequence of graphs with GSOs $(\DD_n)_n$ and associated  GWMs $(\AA_n)_n$. Suppose that $\mathcal{L}$ is an unbounded GRSO such that $G_n \xrightarrow[n \to \infty]{U} \mathcal{L}$. Then, there exists a sequence of permutations $(\pi_n)_n$ such that for every $\lambda > 0$
$$
\| h( P_\mathcal{L}(\lambda) T_{W_{\pi_n({\AA_n})}}  P_\mathcal{L}(\lambda) )  - h(\mathcal{L}  P_\mathcal{L}(\lambda) )  {\|_{L^2[0,1] \to L^2[0,1]}} \xrightarrow{n \to \infty} 0.
$$
\end{lemma}

In Lemma \ref{thm:TransfUnbddL}, the filter is applied to the induced graphon shift operator, projected to the Paley-Wiener space. However, in practical graph deep learning, the filters are applied to the GSO. To bridge the gap between Lemma \ref{thm:TransfUnbddL} and practical graph deep learning, we introduce the following definition.

\begin{definition}
Let $(G_n)_n$ be a sequence of graphs with GWMs $(\AA_n)_n$, $\Theta$  be a family of filters, and $\mathcal{L}$ be an unbounded GSO. We say that
$\Theta$ \emph{approximately commutes} with $P_{\mathcal{L}}$ on $(G_n)_n$ if for every $h \in \Theta$ and every $\lambda > 0$
$$
\| P_\mathcal{L}(\lambda) h( T_{\AA_n}) P_\mathcal{L}(\lambda) -  h\big( P_\mathcal{L}(\lambda) T_{\AA_n} P_\mathcal{L}(\lambda)\big)  {\|_{L^2[0,1] \to L^2[0,1]}} \xrightarrow{n \to \infty} 0.
$$
\end{definition}

 In the following claim, we show that 
the family of monomial filters approximately commutes with $P_\Delta$  in our motivating example -- the classical Laplace operator $\Delta$.

\begin{claim}
\label{claim:LaplaceAsymptCom}
Let $\Delta$ be the Laplace operator from Example \ref{ex:LaplaceOperator}. Suppose that $(\DD_n)_n$ is the finite difference GSO, i.e., 
\[
\DD_n = \begin{bmatrix} 
    -2 & 1 & 0 & \dots & 0 & 1 \\
    1 & -2 & 1 & 0 & \dots & 0 \\
    & \ddots & \ddots & \ddots   \\
        && \ddots & \ddots  & \ddots\\
    0 & \dots & 0 & 1 & -2 & 1    \\
    1 & 0 & \dots & 0 & 1 & -2
    \end{bmatrix}.
\]
Let $(G_n)_n$ be the corresponding sequence of graphs, and $(\AA_n)_n$ the GWMs.
Let $\Theta = \{ z^k  \, | \, k \in \mathbb{N} \}$ be the family of monomial filters.
Then, $\Theta$ approximately commutes with $P_\Delta$ on $(G_n)_n$, i.e.,
for every $\lambda > 0 $ and $k \in \mathbb{N}$,
\[\|P_\Delta(\lambda)T_{\AA_n}^k P_\Delta(\lambda) - \Big(P_\Delta(\lambda)T_{\AA_n} P_\Delta(\lambda)\Big)^k {\|_{L^2[0,1] \to L^2[0,1]}} \xrightarrow{n\rightarrow\infty} 0.\]
\end{claim}

The proof of Claim \ref{claim:LaplaceAsymptCom} is in \ref{Proof:ClaimLaplace}.
The following lemma extends approximate commutations of $\Theta$ with $P_{\mathcal{L}}$ to linear combinations of $\Theta$.

\begin{lemma}
\label{lemma:apprComSPan}
Let $(G_n)_n$ be a sequence of graphs with GWMs $(\AA_n)_n$, $\mathcal{F}$ be a family of filters and $\mathcal{L}$ an unbounded GSO. Suppose that
$\Theta$ approximately commutes with $P_{\mathcal{L}}$ on $(G_n)_n$. Then,
${\rm span } \Theta$ approximately commutes with $P_{\mathcal{L}}$ on $(G_n)_n$.
\end{lemma}

\begin{proof}
Let $h = \sum_{k = 0}^n \alpha_k h_k \in {\rm span } \Theta$. Then, for $\lambda > 0$
$$
\begin{aligned}
& \|P_\mathcal{L}(\lambda) h(T_{W_{\AA_n}}) P_\mathcal{L}(\lambda)  - h\Big(P_\mathcal{L}(\lambda)T_{W_{\AA_n}} P_\mathcal{L}(\lambda)\Big) \|_{\L2L} \\
& = \| P_\mathcal{L}(\lambda) \Big(\sum_{k = 0}^n \alpha_k h_k(T_{W_{\AA_n}}) \Big) P_\mathcal{L}(\lambda) - \sum_{k = 0}^n \alpha_k h_k\Big(P_\mathcal{L}(\lambda)T_{W_{\AA_n}} P_\mathcal{L}(\lambda)\Big)\|_{\L2L} \\
& = \| \sum_{k = 0}^n \alpha_k P_\mathcal{L}(\lambda) h_k(T_{W_{\AA_n}}) P_\mathcal{L}(\lambda) - \sum_{k = 0}^n \alpha_k h_k\Big(P_\mathcal{L}(\lambda)T_{W_{\AA_n}} P_\mathcal{L}(\lambda)\Big)\|_{\L2L} \\
& \leq \sum_{k = 0}^n |\alpha_k| \|  P_\mathcal{L}(\lambda) h_k(T_{W_{\AA_n}}) P_\mathcal{L}(\lambda) - h_k\Big(P_\mathcal{L}(\lambda)T_{W_{\AA_n}} P_\mathcal{L}(\lambda)\Big)\|_{\L2L}\\
& \quad  \xrightarrow{n \to \infty} 0.
\end{aligned}
$$
\end{proof}

The following proposition shows that filters that approximately commute with spectral projections of an unbounded GRSOs are transferable between graphs.

\begin{theorem}
\label{prop:convunbdd}
 Let $(G_n)_n$  be a sequence of graphs with GSOs $(\DD_n)_n$ and associated GWMs $(\AA_n)_n$. Suppose that $\mathcal{L}$ is an unbounded GRSO such that $G_n \xrightarrow[n \to \infty]{U} \mathcal{L}$. 
 Let $\Theta$ be a family of filters with $h(0) = 0$ for every $h\in\Theta$. Suppose that $\Theta$ approximately commutes with $P_{\mathcal{L}}$ on $(G_n)_n$.
 Then, there exists a sequence of permutations $(\pi_n)_n$ such that for every $\lambda > 0$ and every $h\in {\rm span} \Theta$
\[
\| P_\mathcal{L}(\lambda) T_{W_{n h(\pi_n({\DD_{n}}))}}  P_\mathcal{L}(\lambda)  - P_\mathcal{L}(\lambda) T_{W_{m h(\pi_m({\DD_{m}}))}}  P_\mathcal{L}(\lambda) {\|_{L^2[0,1] \to L^2[0,1]}} \xrightarrow{ n, m \to \infty} 0.
\]
\end{theorem}
\begin{proof}
Let $h \in {\rm span}\Theta$  and $(\pi_n)_n$ be the sequence that we get from Lemma \ref{thm:TransfUnbddL}. It holds $h(0) = 0$.
We use the triangle inequality to achieve
$$
\begin{aligned}
& \| P_\mathcal{L}(\lambda) T_{W_{n h(\pi_n({\DD_{n}}))}}  P_\mathcal{L}(\lambda)  - h(\mathcal{L} ) P_\mathcal{L}(\lambda) {\|_{L^2[0,1]}}
\\
&  \leq \| P_\mathcal{L}(\lambda) T_{W_{n h(\pi_n({\DD_{n}}))}}  P_\mathcal{L}(\lambda)  - h( P_\mathcal{L}(\lambda) T_{W_{\pi_n({\AA_n})}}  P_\mathcal{L}(\lambda) ) {\|_{L^2[0,1]}} \\
& + \|  h( P_\mathcal{L}(\lambda) T_{W_{\pi_n({\AA_n})}}  P_\mathcal{L}(\lambda) )  - h(\mathcal{L} )P_\mathcal{L}(\lambda){\|_{L^2[0,1]}}.
\end{aligned}
$$
Since ${\rm span }\mathcal{F}$ approximately commutes with $P_\mathcal{L}$ on $G_n$, by Lemma \ref{thm:TransfUnbddL} and Lemma \ref{prop:graphdomain1}, the last term goes to zero for $n \to \infty$. The proof follows by application of the triangle inequality
\end{proof}

As an example application of Theorem \ref{prop:convunbdd} by Claim \ref{claim:LaplaceAsymptCom} and Lemma \ref{lemma:apprComSPan}, polynomial filters are transferable between graphs with different finite difference GSO.



\subsection{Asymptotic Transferability of Graph Neural Networks}
In this subsection, we prove transferability of SCNNs between graphs approximating the same graphon. 
GCNNs and GCNNons are nonlinear operators, so we do not measure their repercussion error in operator norm. Instead, we analyze the repercussion error of GCNNs and GCNNons point-wise on feature maps. For this, we give the following definitions.

\begin{definition}
Let $d \in \mathbb{N}$ and $y \in  {L^2[0,1]}^d$ be a graphon feature map and $\mathbf{x} \in \mathbb{C}^{n \times d}$ a graph feature map.  
\begin{itemize}
    \item[$(1)$]  We define 
\begin{equation}
\label{eq:l2vectornorm}
{\| y  \|_{L^2[0,1]^d}}= \max_{k=1, \ldots, d} {\| y ^k \|_{L^2[0,1]}}.
\end{equation}
\begin{equation}
\| \mathbf{x}\|_{\mathbb{C}^{n \times d}} = \max_{k=1, \ldots, d}\| \mathbf{x} ^k \|_{\mathbb{C}^n}.
\end{equation}
    \item[$(2)$]
    We say that $y$ is \emph{normalized} if $\max_{k = 1, \ldots, d}  \|y^k  {\|_{L^2[0,1]}} \leq 1$. We say that $\mathbf{x}$ is \emph{normalized} if $\max_{k = 1, \ldots, d} \|\mathbf{x}^k\|_{\mathbb{C}^n } \leq 1$.

    \item[$(3)$] Let $(G_n)_n$ be a sequence of graphs with GSOs and $(\mathbf{x}_n)_n$ a sequence of graph feature maps with $\mathbf{x}_n \in \mathbb{C}^{n \times d}$ for every $n \in \mathbb{N}$.
    Let $\mathbf{x}_n^k$ and $y^k$ denote the $k$-th feature  of $\mathbf{x}_n$ and $y$.
    We say that $(\mathbf{x}_n)_n$ converges to $y$  if for every $k = 1, \ldots, n$ it holds $\mathbf{x}_n^k \xrightarrow[n \to \infty]{I} y^k$, (see  Definition \ref{def:convergenceOfSignal}).
    We  write in this case
    $\mathbf{x}_n \xrightarrow[n \to \infty]{I} y$.
\end{itemize}
\end{definition}

The following lemma shows the transferability of SCNN between graphons and induced graphons from a convergent sequence of graphs.

\begin{lemma}
\label{cl:asympTransferabilityGCNNon}
Let $\phi:= (H, \mathbf{M},\rho)$ be an SCNN with $L$ layers and Lipschitz continuous activation function $\rho: \mathbb{R} \rightarrow \mathbb{R}$. Let $F_l$ be the number of features in the $l$-th layer, for $l=1,\ldots,L$. Let $(G_n)_n$ be a sequence of graphs with GSOs $(\DD_n)_n$ and associated  GWMs $(\AA_n)_n$. Let $W \in \mathcal{W}$ satisfy
$G_n \xrightarrow[n \to \infty]{H} W$. 
Suppose that $(y_n)_n$ is a sequence in ${(L^2[0,1])^{F_0}}$ and $y \in  {(L^2[0,1])^{F_0}}$ 
such that
\begin{equation}
\label{eq1:cl:asym}
\|y_n - y   {\|_{L^2[0,1]^{F_0}}} \xrightarrow{ n \to \infty} 0.
\end{equation}
 Then there exists a sequence of permutations $(\pi_n)_n$ such that, 
\begin{equation}
    \|\phi_{W_{\pi_n(\AA_n)}}(y_n)  - \phi_W(y)  {\|_{L^2[0,1]^{F_l}}} \xrightarrow{n \to \infty} 0.
\end{equation}
\end{lemma}

The proof of Lemma \ref{cl:asympTransferabilityGCNNon} is given in \ref{proof:asympTGCNNon}.
The following Lemma is a generalization of Proposition~\ref{prop:graphdomain2} to SCNNs.

\begin{lemma}
\label{lemma:GCNNdomain}
Let $\phi:=\phi(H, \mathbf{M}, \rho)$ be an SCNN. Let $G$ be a graph with GSO $\DD$ and associated GWM $\AA$. Let $\mathbf{x} \in \mathbb{C}^{n \times d}$.
Then, 
$
\psi_{\phi_{{\DD}  }( \mathbf{x})} = \phi_{W_{\AA}}(\psi_\mathbf{x})
$.
\end{lemma}
\begin{proof}
We prove by induction over the layers $l=1, \ldots, L$. Let $l=1$, and $j = 1, \ldots, F_1$.
It holds $$
\phi^1_\Delta(\mathbf{x})^j = \rho \Big( \sum_{k=1}^{F_{0}} m^{jk}_{1} h^{jk}_{1}(\Delta) \mathbf{x}^k \Big) = \rho ( \tilde{\phi}^1_\Delta(\mathbf{x})^j).
$$
We induce the resulting signal to the graphon space, and obtain
\begin{equation}
\begin{aligned}
    \psi^1_{ \rho ( \tilde{\phi}_\Delta(\mathbf{x})^j)}
    & = \sum_{i=1}^n \chi_{P_i} (\cdot) \rho ( \tilde{\phi}^1_\Delta(\mathbf{x})_i^j)
    = \rho \Big(\sum_{i=1}^n \chi_{P_i} (\cdot) ( \tilde{\phi}^1_\Delta(\mathbf{x})_i^j) \Big) \\
    & = \rho (\psi^1_{ \tilde{\phi}_\Delta(\mathbf{x})^j} )
     = \rho\Big(\sum_{k=1}^{F_{L-1}} \psi_{m_1^{jk}h_1^{jk}(\Delta) \mathbf{x}^k}\Big) = \rho\Big(\sum_{k=1}^{F_{0}} {m_1^{jk} h_1^{jk}(T_{W_{\AA}}) \psi_{\mathbf{x}^k}}\Big) \\ 
     & =  \phi^1_{W_{\AA}}(\psi_\mathbf{x})^j.
    \end{aligned}
\end{equation}
The third-to-last equality holds since the induction operator commutes with the addition of filters.
The second-to-last equality holds by Proposition \ref{prop:graphdomain2}. The proof follows by a similar computation for the other layers.
\end{proof}

The following result is a generalization of Theorem
\ref{cor:ConvGraphFilter} to GCNNs.

\begin{theorem}
\label{thm:GCNNtrans1}
Consider the same conditions as in Lemma~\ref{cl:asympTransferabilityGCNNon}. Let $(\mathbf{x}_n)_n$ be a sequence of graph feature maps, and $y \in  { (L^2[0,1])^{F_0} }$, such that $\mathbf{x}_n \xrightarrow[n \to \infty]{I} y$.
 Then, there exists a sequence of permutations $(\pi_n)_n$ such that
\begin{equation*}
    \| \psi_{\phi_{\pi_n({\DD_n})}(\mathbf{x}_n)} - \psi_{\phi_{\pi_m({\DD_m})}(\mathbf{x}_m)}  {\|_{L^2[0,1]^{F_L}}} \xrightarrow{n,m \to \infty} 0.
\end{equation*}
\end{theorem}
\begin{proof}
By Lemma~\ref{cl:asympTransferabilityGCNNon} and Lemma~\ref{lemma:GCNNdomain}, there exists a sequence of permutations $(\pi_n)_n$ such that
$$
\| \psi_{\phi_{\pi_n({\DD_n})}(\mathbf{x}_n)} - \phi_W(y)  {\|_{L^2[0,1]^{F_L}}} \xrightarrow{n \to \infty} 0.
$$
The proof now follows by the triangle inequality.
\end{proof}
\subsection{Non-Asymptotic Transferability of Graph Neural Networks}
In this section, we derive stability bounds for GCNNs. We restrict ourselves to \emph{contractive} activation functions.

\begin{definition}
An activation function $\rho: \mathbb{R} \rightarrow \mathbb{R}$ is called \emph{contractive} if $|\rho(x) - \rho(y)| \leq |x-y|$ for every $x,y \in \mathbb{R}$. 
\end{definition}

The following lemma shows that GCNNons, based on filters with higher regularity, achieve linear approximation rate.

\begin{lemma}
\label{claim:ConvOfGrNN}
Let $\phi:=\phi(H, \mathbf{M},\rho)$ be an SCNN 
with  contractive activation function $\rho$. Let $G$ be a graph with GSO $\DD$ and associated GWM $\AA$. Let $W \in \mathcal{W}$ and suppose that
\[\|\AA\|_{\mathbb{C}^n \rightarrow \mathbb{C}^n}\leq \Gamma,\ \  \|T_W  {\|_{L^2[0,1] \rightarrow L^2[0,1]} }\leq \Gamma\] 
for $\Gamma > 0$. Let $\varepsilon > 0$, and suppose  there exists a node permutation $\pi$ of $G$ satisfying
\[\|T_{W_{\pi(\AA)}} - T_W  {\|_{L^2[0,1]\rightarrow L^2[0,1]}} \leq \varepsilon.\]
Let $\mathbf{x}$ be a normalized graph feature map  and $y$  a graphon feature map satisfying 
\[
 \|\psi_\mathbf{x} - y  {\|_{L^2[0,1]^{F_0}} } \leq \varepsilon.\]
Suppose that for every filter $h$ in the SCNN $\phi$, we have $h \in C^1[-\Gamma, \Gamma]$ with Lipschitz continuous derivative and $\|h\|_{L^\infty[-\Gamma, \Gamma]} \leq 1$.
Then, there exists a $C_L \in \mathbb{N}$ such that
\begin{equation}
\label{GCNNspeed:eq}
\| \phi_{W_{\pi(\AA)}} (\psi_\mathbf{x}) - \phi_{W} (y)  {\|_{L^2[0,1]^{F_L}}}\leq C_L \varepsilon,
\end{equation}   
where $C_L$ is a constant that depends on $\mathbf{M}$ and $H$, and is described in Remark \ref{remark:GCNNconstant}.
%
\end{lemma}
The proof of Lemma \ref{claim:ConvOfGrNN} is in  \ref{proof:LinConvGCNN}.

\begin{remark} 
\label{remark:GCNNconstant}
Consider the setting of Lemma \ref{claim:ConvOfGrNN}. For every $l= 1, \ldots, L$, let $C_l^{jk}$ denote the constant $C$ of (\ref{eq:LinConv}) corresponding to the filter $h=h_l^{jk}$, for $j = 1, \ldots, F_l$, $k = 1, \ldots, F_{l-1}$.  Let $C=\max_{l,j,k} C_l^{jk}$, where the maximum is over $l = 1,\ldots, L$, $j = 1, \ldots, F_l$, and $k = 1, \ldots, F_{l-1}$.
Further, let $M = \max_{l = 1, \ldots,L} \|\mathbf{M}_l\|_\infty$. Then, the constant $C_L$ in (\ref{GCNNspeed:eq}) is given by 
\begin{equation}
\label{eq:constantGCNN}  
C_L = M^L(1 + L C).
\end{equation}
\end{remark}

\begin{remark}
In Lemma \ref{claim:ConvOfGrNN}, the constant $C_L$ can increase exponentially in $L$.
We can control $C_L$ by demanding $\|\mathbf{M}_l\|_\infty \leq 1$ for every $l=1,\ldots, L$, which gives
\[
 \| \phi_{W_{\pi(\AA)}} (\psi_\mathbf{x})- \phi_{W} (y)   {\|_{L^2[0,1]^{F_L}}} \leq (1+LC) \varepsilon.
\]
The condition $\|\mathbf{M}_l\|_\infty \leq 1$ is relaxed to weight decay regularization in practical deep learning. Hence (\ref{eq:constantGCNN}) indicates that weight decay might be an important regularization  for promoting transferability in practical deep learning. A similar observation was given in \cite{LevieTransfSpectralGraphFiltersLong}.
\end{remark}

The following theorem shows GCNNs achieve linear approximation rate.
\begin{theorem}
\label{thm:GCNNtrans2}
Let $\phi:=\phi(H, \mathbf{M},\rho)$ be an SCNN
with  contractive activation function $\rho$. 
Let $G_1, G_2$ be graphs with GSOs $\DD_1, \DD_2$ and associated GWM $\AA_1, \AA_2$.
Suppose that
\[\|\AA_1\|_{\mathbb{C}^n \rightarrow \mathbb{C}^n}\leq \Gamma,\ \  \|\AA_2\|_{\mathbb{C}^m  \rightarrow \mathbb{C}^m } \leq \Gamma\] 
for $\Gamma > 0$. Let $\varepsilon > 0$, and suppose  there exist node permutations $\pi_1$ of $G_1$ and $\pi_2$ of $G_2$ satisfying
\[\|T_{W_{\pi_1(\AA_1)}} - T_{W_{\pi_2(\AA_2)}}  {\|_{L^2[0,1]\to L^2[0,1]}} \leq \varepsilon.\]
Let $\mathbf{x}_1$ and  $\mathbf{x}_2$ be normalized graph feature maps satisfying 
\[
 \|\psi_{\mathbf{x}_1} - \psi_{\mathbf{x}_2}  {\|_{L^2[0,1]^{F_0}}} \leq \varepsilon.\]
Then, 
$$   
\| \psi_{\phi_{{\pi(\DD_1)}  }( \mathbf{x_1})}- \psi_{\phi_{{\pi(\DD_2)}  }( \mathbf{x_2})}  {\|_{L^2[0,1]^{F_L}}} \leq C_L \varepsilon,
$$
where $C_L$ is given in (\ref{GCNNspeed:eq}). 
\end{theorem}

The proof of Theorem \ref{thm:GCNNtrans2} follows directly from Lemma \ref{claim:ConvOfGrNN} and Lemma \ref{lemma:GCNNdomain}.

\section{Conclusions}
 In this paper, we proved that continuous graph filters and end-to-end graph neural networks with continuous filters are transferable in the graphon sense. 
Generally, a filter is said to be transferable if it has approximately the same repercussion on graphs that represent the same phenomenon.
 To specify the definition of transferability in the graphon sense, we used the space of graphons as a latent space from which graphs are sampled, and to which graphs are embedded, via induction, in order to compare them. In our setting, graphs represent the same phenomenon if they approximate the same graphon in homomorphism density.

 In Theorem \ref{cor:DiscrFilterConv}  we showed that continuous filters are transferable, were the repercussion error is measured in the induced graphon sense.
In Theorem \ref{thm:GCNNtrans1} we showed that any GCNN with continuous filters is transferable, where the repercussion error is measured, for a given graph signals, in the induced signal sense. We provide the first work to consider graphon-based transferability for GCNNs with non-polynomial filters, which are used in GCNN architectures like \cite{levie2018cayleynets, armaFilters}.

 In Theorem \ref{convRateC1filter} and Theorem \ref{thm:GCNNtrans2} we introduced approximation rates for the repercussion error of filters and GCNNs with filters that have Lipschitz continuous derivative. We showed that the repercussion errors are linearly stable with respect to $\|T_{W_{n_1(\DD_1)}} - T_{W_{n_2(\DD_2)}}  {\|_{L^2[0,1]\to L^2[0,1]}} $. The bound from Theorem \ref{thm:GCNNtrans2} may hint to the importance of  weight decay regularization in graph deep learning.
 
 Furthermore, we introduced a framework for analyzing graphs that approximate unbounded GRSOs.  
 In Subsection \ref{subsec:unbounded} we show that SCNNs on Euclidean spaces with the Laplace operator are classical CNNs. 
 Discretizations of different resolutions of the same Euclidean space by grids are graphs that represent the same phenomenon. 
 To show transferability between classical discrete Euclidean signals on grids of different resolutions, we must be able to work with unbounded GRSOs, since the classical Laplace operator is unbounded. Our analysis allows this,  showing that our framework of graph transferability aligns with the stability of classical CNNs between resolutions.
\appendix
\section{Proofs}

\subsection{Proof of Proposition \ref{prop:LaplaceOperatorAppro}}
\label{proof:LaplaceOperatorAppro}
Let $(\lambda_m)_m$ be a sequence of bands such that $\lambda_m \xrightarrow{m \to \infty} \infty$. 
By Definition \ref{def:unboundedGSRO} there exists a graphon $W^{\lambda_m} \in \mathcal{W}$ such that $\mathcal{L} P_\mathcal{L}(\lambda_m) = T_{W^{\lambda_m}}$ for every $m \in \mathbb{N}$.
Then, for every $m \in \mathbb{N}$,   there exists a sequence of graphs with GSO $(G^m_{n})_n$, where $n$ denotes the number of nodes, such that  $G^m_{n} \xrightarrow[n \to \infty]{H} W^{\lambda_m}$. 
For every $m \in \mathbb{N}$, let $(\AA^m_n)_n$ be the  sequence of GWMs associated to $(G^m_{n})_n$. 
Theorem~\ref{thm:cutnormconv1} and Theorem~\ref{lemma:cutnormeqschattenpnorm} imply that,
after adequate relabeling of the graph nodes,  $\| T_{W_{\AA^m_{n}}} - T_{ W^{\lambda_m}}\|_{L^2[0,1] \to L^2[0,1]} \xrightarrow{n \to\infty} 0 $. Then, for every $m \in \mathbb{N}$ 
$$
 \| T_{W_{\AA^m_{n}}} P_\mathcal{L}(\lambda_m) - \mathcal{L} P_\mathcal{L}(\lambda_m) \|_{L^2[0,1] \to L^2[0,1]} \xrightarrow{n \to \infty} 0.
$$
We choose for every $m \in \mathbb{N}$ an $n_m \in \mathbb{N}$  such that
$$
\| T_{W_{\AA^m_{n_m}}} P_\mathcal{L}(\lambda_m) - \mathcal{L} P_\mathcal{L}(\lambda_m) \|_{L^2[0,1] \to L^2[0,1]} \leq 1/m.
$$
We denote by abuse of notation $G_m := G^m_{n_m}$ and consider the sequence $(G_m)_m$ with GWMs $(\AA_m)_m$.
{It is now easy to see that $(G_m)_m$ converges to $\mathcal{L}$ in the sense of Definition \ref{def:unbddconv}.}

\subsection{Proof of Lemma~\ref{claim:ConvOfGrFltBdd}}
\label{Appendix:Graphons}
We start this subsection by defining the \emph{cut distance} -- a pseudo-metric in the space of graphons.  As shown in Theorem \ref{thm:homdensityLipWRTcutdistance} below, convergence in cut distance is equivalent  to convergence in homomorphism density. 
Borgs et al. defined the cut distance, based on the cut norm,  as a permutation invariant distance (see \cite{Borgs2007graph}). 
For this, denote by $S_{[0,1]}$  the set of all bijective measure preserving maps between the unit interval and itself.
\begin{definition}
Let $W, U \in \mathcal{W}$ be two graphons. The \emph{cut distance} between $W$ and $U$ is defined by
\begin{equation}
    \delta_{\square} (W, U) := \inf_{f \in S_{[0,1]}} \|W - U^f\|_{\square},
\end{equation}
where $U^f(x,y) = U(f(x), f(y))$. 
\end{definition}

The following theorem from~\cite[Theorem 3.7 ]{Borgs2007graph} shows that convergence in homomorphism density is equivalent to convergence in cut distance. 
\begin{theorem}[{{\cite[Theorem 3.7 ]{Borgs2007graph}}}]
\label{thm:homdensityLipWRTcutdistance}
Let $U,W \in \mathcal{W}$ and $C = \max \{ 1, \|W\|_{\infty}, \|U\|_{\infty}\}$.
\begin{itemize}
\item[$(1)$]  For a simple graph $F$ with $m$ edges
$$
|t(F,U) - t(F,W)| \leq 4 m C^{m-1}\delta_{\square}(U,W).
$$
\item[$(2)$] If $|t(F,U) - t(F,W)| \leq 3^{-k^2}$ for every simple graph $F$ with $k $ nodes, then
$$
\delta_{\square}(U,W) \leq \frac{22C}{\sqrt{log_2k}}.
$$
\end{itemize}
\end{theorem}
Note that
$
t(F,G) = t(F, W_{\AA})
$
 for every graph $G$ with GWM $\AA$. 
 As a result,  whenever we have a sequence of graphs $(G_n)_n$ with GWMs $(\AA_n)_n$ and $W \in \mathcal{W}$ such that $G_n \xrightarrow[n \to \infty]{H} W$, then $W_{\AA_n} \xrightarrow[n \to \infty]{H} W$ as well. 
 Theorem~\ref{thm:homdensityLipWRTcutdistance} implies that this is true if and only if
$$
\delta_{\square}(W_{\AA_n},W) \xrightarrow{n \to \infty} 0.
$$

  {Lov\'{a}sz and  Szegedy introduced in~\cite{10.1016/j.jctb.2006.05.002} the \textit{cut norm} on the space of graphons  as 
\begin{equation}
\label{def:cutnorm}
    \| W\|_{\square} :=   {\sup_{S, T \subset [0,1]}} | \int_{S \times T}W(u,v) d\mu(u)d \mu(v) |,
\end{equation} 
taking the supremum over all measurable sets $S$ and $T$. }
{
The following result, taken from~\cite[Lemma 5.3]{Borgs2007graph}, shows that convergence in homomorphism density is closely related to convergence in cut norm. 
}
 
 {
\begin{theorem}[{{\cite[Lemma 5.3]{Borgs2007graph}}}]
\label{thm:cutnormconv1}
Let $(G_n)_{n}$ a sequence of graphs with GSOs $(\DD_n)_n$ and $(\AA_n)_n$ be the corresponding sequence of  GWMs. Let $W \in \mathcal{W}$ satisfy $G_n \xrightarrow[n\to\infty]{H} W$. Let $(W_{\AA_n})_n$  be the sequence of graphons induced by $(G_n)_{n}$. Then, there exists a sequence of permutations $(\pi_n)_{n}$ such that 
{$
\| {W}_{\pi_n(\AA_n)} -  W\|_{\square} \xrightarrow{n \to \infty} 0
$}.
\end{theorem}
}
{
 Theorem \ref{thm:cutnormconv1} together with the following lemma  show that convergence in homomorphism density is indeed equivalent (up to relabelling) to convergence of the induced GRSOs in operator norm and  Schatten $p$-norms.
}

{
\begin{lemma}
\label{lemma:cutnormeqschattenpnorm}
For any $2 < p < \infty$ and $\Gamma > 0$, if $W \in \mathcal{W}_{\Gamma}$, then
\begin{equation}
\label{eq:cutnormeqschattenpnorm}
\|W\|_{\square} \leq \|T_W\|_{\L2L} \leq \|T_W\|_{S_p}
\leq    \sqrt{2} \Gamma^{1/2 + 1/p} \|W\|_{\square}^{1/2 -  1/p} .
\end{equation}
\end{lemma}
}
{ 
Lemma \ref{lemma:cutnormeqschattenpnorm} follows directly from \cite[Lemma E.7(i)]{Janson2010GraphonsCN}, in which $W$ is assumed to be in $\mathcal{W}_1$. To prove Lemma \ref{lemma:cutnormeqschattenpnorm} we simply rescale the graphon $W$ so it is in $\mathcal{W}_1$, and use \cite[Lemma E.7(i)]{Janson2010GraphonsCN}. }

Let $\mathcal{H}$  be a Hilbert space. We denote by $\mathcal{B}(\mathcal{H})$ the space of bounded linear operators on $\mathcal{H}$. We denote the spectrum of the operator $A \in \mathcal{B}(\mathcal{H})$  by $\sigma(A)$. 
 The following result shows that, for every filter $h$, the operator norm of a realization of $h$ is equal to the infinity norm of $h$   \cite[Satz VII.1.4]{Funkana}.
\begin{lemma}[{{\cite[Satz VII.1.4]{Funkana}}}]
\label{lemma:FuncCalcBounded}
 Let $\mathcal{H}$ be a Hilbert space and $A \in \mathcal{B}(\mathcal{H})$ be normal. Let $h:\mathbb{R} \rightarrow \mathbb{C}$ be a continuous function, then  $\|h(A)\|_{\mathcal{H} \rightarrow \mathcal{H}} =  \|\chi_{\sigma(A)}h \|_{\infty}$, where $\chi_{\sigma(A)}$ is the indicator function of $\sigma(A)$. 
\end{lemma}

The following lemma shows that, given a convergent sequence of bounded operators on a Hilbert space, the continuous functional calculus preserves the convergence.

\begin{lemma}
\label{lemma:contoffunccalc}
Let $\mathcal{H}$ be a Hilbert space and $\mathcal{S}(\mathcal{H}) \subset \mathcal{B}(\mathcal{H})$ be the space of bounded self-adjoint operators  with the operator norm topology. Let $h: \mathbb{R} \rightarrow \mathbb{C}$ be a continuous function. Then, the mapping
\begin{equation}
\label{eq1:lemmaContofFuncCalc}
     \mathcal{S}(\mathcal{H})  \rightarrow \mathcal{B}(\mathcal{H}), \quad 
    A  \mapsto h(A),
\end{equation}
is continuous.
\end{lemma}
 
 {Lemma \ref{lemma:contoffunccalc} can be easily verified, since the continuity of (\ref{eq1:lemmaContofFuncCalc}) is trivial for polynomials $h$, and follows for continuous functions from the Weierstrass theorem (\cite[Section 8]{aleksandrov2010operator}).}

\begin{proof}[Proof of Lemma~\ref{claim:ConvOfGrFltBdd}]
Let $(G_n)_n$ be  a sequence of graphs with GSOs $(\DD_n)_n$ and associated sequence of GWMs $(\AA_n)_n$ such that $G_n \xrightarrow[n \to \infty]{H} W$. Then, by Theorem~\ref{thm:cutnormconv1} and  Lemma~\ref{lemma:cutnormeqschattenpnorm}, there exists a sequence of permutation $(\pi_n)_n$ such that
\begin{equation*}
    \|T_{W_{\pi_n(\AA_n)}} - T_W \|_{\L2L} \xrightarrow{n \to \infty} 0.
\end{equation*}
Let $h: \mathbb{R} \rightarrow \mathbb{R}$ be continuous. Then, by Lemma~\ref{lemma:contoffunccalc}, 
\begin{equation*}
      \|h(T_{W_{\pi_n(\AA_n)}}) - h(T_W) \|_{\L2L} \xrightarrow{n \to \infty} 0.
\end{equation*}
Let $(\mathbf{x}_n)_n$ be as in the formulation of Lemma~\ref{claim:ConvOfGrFltBdd}. Then, 
\[
\begin{aligned}
        & \|h(T_{W_{\pi_n(\AA_n)}})\psi_{\mathbf{x}_n} - h(T_W)\psi\|_{L^2(\J)} \\
  &  \leq \|h(T_{W_{\pi_n(\AA_n)}})\psi_{\mathbf{x}_n} - h(T_W)\psi_{\mathbf{x}_n}\|_{L^2(\J)} + \|h(T_W)\psi_{\mathbf{x}_n} - h(T_W)\psi\|_{L^2(\J)} \\
  & \leq \|h(T_{W_{\pi_n(\AA_n)}}) - h(T_W)\|_{\L2L}\|\psi_{\mathbf{x}_n}\|_{L^2(\J)} \\
  &   + \|h(T_W)\|_{\L2L} \|\psi_{\mathbf{x}_n} - \psi\|_{L^2(\J)} \xrightarrow{n \to \infty} 0.
\end{aligned}
\]
\end{proof}

\subsection{Proof of Theorem \ref{thm:LinConvHSnorm}}
\label{subsec:proofConvRateHSNorm}
Let $(G_n)_n$ be a sequence of graphs such that $G_n \xrightarrow[n \to \infty]{H} W$. 
Then, by Theorem \ref{thm:cutnormconv1} and Lemma \ref{lemma:cutnormeqschattenpnorm}, there exists a sequence of permutations $(\pi_n)_n$ such that
$
\|T_{W_{\pi_n(\AA_n)}} - T_W \|_{S_p} \xrightarrow{n \to \infty} 0
$. 
Every Lipschitz continuous function $h$ with Lipschitz constant $\|h\|_{Lip}$ satisfies
$$
\|h(A) - h(B)\|_{S_p} \leq \|h\|_{Lip} {K_p} \|A - B\|_{S_p}
$$
for  self-adjoint operators $A$ and  $B$ {such that $\|A-B\|_{S_p} < \infty$}, see~\cite[Theorem 1]{OLinSchattenNeumannClasses}.  Hence, 
$$
\begin{aligned}
\|h(T_{W_{\pi_n(\AA_n)}}) - h(T_W) \|_{p}  & \leq {K_p} \|h\|_{Lip} \|T_{W_{\pi_n(\AA_n)}} - T_W \|_{p}.
\end{aligned}
$$
Using Proposition~\ref{prop:graphdomain1} and the triangle inequality we finish the proof. 

\subsection{Proof of Theorem~\ref{convRateC1filter}}
\label{subsec:proofConvRate}
We first recall basic results from Fourier analysis. Let $\gamma > 0$. 
The space $L^2(\gamma\mathbb{T})$  is defined as the space of
$\gamma$-periodic functions
$
h: \mathbb{R} \rightarrow \mathbb{R}
$ with $\|h\|_{L^2[0,\gamma]} < \infty$.
 For  $L^2(\gamma\mathbb{T})$, the Fourier coefficients $\hat{h}(n)$ are defined for every $n \in \mathbb{Z}$ by 
$$
\hat{h}(n) = \frac{1}{\gamma} \int_{-\gamma/2}^{\gamma/2} h(t) e^{-2 \pi int/ \gamma}.
 $$
The $N$-th partial sum of the Fourier series of $h$ is defined by
 \[
 (S_Nh)(t) :=  \sum_{n = -N}^{N} \hat{h}(n) e^{2 \pi int/ \gamma}.
 \]
The following lemma is taken from~\cite[Example 5]{OLfcts}.
\begin{lemma}[{{\cite[Example 5]{OLfcts}}}]
\label{lemma:expIsOL}
Let $a \in \mathbb{R}$ and $A$ and $B$ be a pair of bounded self-adjoint operators on the Hilbert space $\mathcal{H}$. Then,
\begin{equation}
    \|e^{iAa} - e^{iBa}\|_{\mathcal{H} \to \mathcal{H}} \leq  |a| \|A - B\|_{\mathcal{H} \to \mathcal{H}}.
\end{equation}
\end{lemma}

The following lemma is a step in proving Theorem~\ref{convRateC1filter}.
\begin{lemma}
\label{thm:linearConvWithFourierSeries}
Let $\gamma > 0$ and   $\varepsilon > 0$. Let $h \in L^2[-\gamma/2, \gamma/2]$  such that 
$(\hat{h}(n)n)_n \in l^1(\mathbb{Z})$.
Let $A$ and $B$ be  self-adjoint operators on a Hilbert space $\mathcal{H}$, satisfying \newline 
$
 \|A -B\|_{\mathcal{H} \to \mathcal{H}} \leq \varepsilon
$ 
for some $\varepsilon > 0$. Then,
\begin{equation}
    \|h(A) - h(B)\|_{\mathcal{H} \to \mathcal{H}} \leq C \varepsilon,
\end{equation}
where $C = 2 + 2 \pi/ \gamma \|(\hat{h}(n)n)_n\|_{l^1(\mathbb{Z})} $.
\end{lemma}
\begin{proof}
Let $\varepsilon > 0$. By the fact that $(\hat{h}(n)n)_n \in l^1(\mathbb{Z})$, we have \newline  
$
\| h - S_nh \|_\infty \xrightarrow{n \to \infty} 0
$. 
We choose $N \in \mathbb{N}$ large enough such that 
$
\| h - S_Nh \|_{\infty} \leq \varepsilon
$. 
   We evaluate $S_N h(A)$ and $S_N h(B)$ and bound their difference in operator norm. By  Lemma~\ref{lemma:expIsOL},
\begin{equation*}
\begin{aligned}
    \|(S_Nh)(A) - (S_Nh)(B)\|_{\mathcal{H} \to \mathcal{H}}  & = \|\sum_{n = -N}^N \hat{h}(n) e^{2 \pi in A/\gamma} - \sum_{n = -N}^N \hat{h}(N) e^{2 \pi in B/\gamma}\|_{\mathcal{H} \to \mathcal{H}} \\
    & \leq \sum_{k = -N}^N |\hat{h}(n)| |n|  2\pi/\gamma  \| A - B \|_{\mathcal{H} \to \mathcal{H}} \\
    & \leq C 2\pi/\gamma  \| A- B\|_{\mathcal{H} \to \mathcal{H}},
\end{aligned}
\end{equation*}
where $C := \|(\hat{h}(n)n)_n\|_{l^1(\mathbb{N})}$. We summarize,
\begin{equation*}
    \begin{aligned}
    \|h(A) - h(B)\|_{\mathcal{H} \to \mathcal{H}} & \leq \|h(A) - (S_Nh)(A)\|_{\mathcal{H} \to \mathcal{H}} + \|(S_Nh)(A) - (S_Nh)(B)\|_{\mathcal{H} \to \mathcal{H}}  \\
    & +  \| (S_Nh)(B) - h(B) \|_{\mathcal{H} \to \mathcal{H}} \\
    & \leq \|h - S_Nh\|_{\infty} + \|(S_Nh)(A) - (S_Nh)(B)\|_{\mathcal{H} \to \mathcal{H}}  + \|h - S_Nh\|_{\infty}  \\
    & \leq \varepsilon +  C2\pi/\gamma  \| A- B\|_{\mathcal{H} \to \mathcal{H}}   + \varepsilon \leq (2 + 2 \pi  C/\gamma) \varepsilon,
    \end{aligned}
\end{equation*}
where the second inequality holds by Lemma \ref{lemma:FuncCalcBounded}.
\end{proof}

We continue by formulating sufficient conditions for  filter functions $h \in L^2[-\Gamma, \Gamma]$ to fulfil the assumptions of Lemma~\ref{thm:linearConvWithFourierSeries}.  

\begin{definition}
A \emph{modulus of continuity} is a function
$
w: [0, \infty] \rightarrow [0, \infty].
$
A function $f$ admits $w$ as a modulus of continuity if
$$
|f(x) - f(y)| \leq w(|x-y|),
$$
for all $x$ and $y$ in the domain of $f$.
\end{definition} 

The following theorem, given in~\cite[p.21 ff.]{JacksonDunham1951Ttoa}, provides us with a sufficient condition for uniform convergence of the Fourier series of a function. 
\begin{theorem}[{{\cite[p.21 ff.]{JacksonDunham1951Ttoa}}}]
\label{thm:ApprSpeedFourierSumModCont}
Let $h \in C^p(\mathbb{R})$ be such that $h^{(p)}$ has modulus of continuity $w$. Suppose that $h$ is of period $\gamma > 0$ 
Then, for the partial sum $S_nh$  of the Fourier series of $h$,
\begin{equation}
\label{eq:FourierseriesNthPartialSumUnifConv}
    \|h - S_nh\|_\infty \leq K \frac{ln (n)}{n^p}w(\gamma / n).
\end{equation}
where, $K \in \mathbb{R}$ denotes an universal constant.
\end{theorem}

Next we recall the \emph{Wiener algebra} \cite{katznelson_2004}.
 \begin{definition}
The \emph{Wiener algebra}, denoted by $A(\gamma \mathbb{T})$, is defined as the space of all $f \in L^2(\gamma \mathbb{T})$ with absolutely convergent Fourier series. 
 \end{definition}

The following result  can be found in \cite{katznelson_2004}.
 \begin{lemma}[{{\cite{katznelson_2004}}}]
 \label{remark:WienerAlgebra}
     Let $f \in L^2(\gamma \mathbb{T})$. Suppose that one of the following conditions is fulfilled.
     \begin{itemize}
         \item[$(1)$] $f \in C^1$.
         \item[$(2)$] $f$ belongs to a $\alpha$-Hölder class for $\alpha > 1/2$. 
         \item[$(3)$] $f$ is of bounded variation and belongs to a $\alpha$-Hölder class for some $\alpha > 0$.
     \end{itemize}
     Then, $f \in A(\gamma \mathbb{T})$. Furthermore,
     \begin{equation}
         A(\gamma \mathbb{T}) \subset C(\gamma \mathbb{T}).
     \end{equation}
 \end{lemma}

We can now prove Theorem~\ref{convRateC1filter}.
\begin{proof}[Proof of Theorem~\ref{convRateC1filter}]
 By assumption, the induced graphon $W_{\AA_i}$ satisfies $\|T_{W_{\AA_i}}\|$ $ \leq \Gamma$, where $i \in \{1,2\}$.  Then,  $h(T_{W_{\AA_i}})$ only depends on the values of $h$ in $[-\Gamma, \Gamma]$. Hence, we can extend $h$ to a function $h_{\rm ext}: [-\Gamma - 1, \Gamma + 1] \rightarrow \mathbb{R}$ such that $h_{\rm ext} (x) = h(x)$ for all $x \in [-\Gamma, \Gamma]$ and $h_{\rm ext}$ admits a periodical extension on the whole real line with Lipschitz continuous derivative. By abuse of notation, we denote the periodical extension of $h_{\rm ext}$ by $h_{\rm ext}$.

There exists a constant $L \in \mathbb{R}$ such that the modulus of continuity $w$ for $h_{\rm ext}'$ is given by 
$ w(t)=L t$.
By Theorem~\ref{thm:ApprSpeedFourierSumModCont}, we have
\begin{equation}
    |h_{\rm ext}(x) - (S_nh_{\rm ext})(x)| \leq \tilde{K} \frac{ln (n)}{n^2}
\end{equation}
 for every $n \in \mathbb{N}$, where  $\tilde{K}= K L 2 (\Gamma+1)$ only depends on $h_{\rm ext}$. The Fourier series of $h_{\rm ext}'$ is given by
$$
h_{\rm ext}'(\zeta ) = \sum_{n \in \mathbb{Z}} \hat{h}_{\rm ext}(n)in/(\Gamma + 1) e^{in\zeta /(\Gamma + 1)}
$$
since $h_{\rm ext}'$ is Lipschitz.
Lipschitz continuity implies bounded variation and 1-Hölder continuity. Hence,  Lemma~\ref{remark:WienerAlgebra} leads to
\begin{equation}
\label{eq:FourierSeriesBoundDerivative}
    \sum_{n \in \mathbb{Z}} |\hat{h}_{\rm ext}(n)||n|  = C < \infty.
\end{equation}
Lemma~\ref{thm:linearConvWithFourierSeries}  leads to 
$$
\|h_{\rm ext}(T_{W_{\pi_1(\AA_1)}}) - h_{\rm ext}(T_{W_{\pi_2(\DD_2)}}) \|_{\L2L} \leq C \varepsilon.
$$
By $h_{\rm ext}(0) = 0$, applying Proposition \ref{prop:graphdomain1} finishes the proof.
\end{proof}

\subsection{Proof of Claim \ref{claim:LaplaceAsymptCom}}
\label{Proof:ClaimLaplace}
\begin{proof}
Let $\lambda > 0$. The Paley-Wiener space $PW_{\Delta}(\lambda)$ is finite-dimensional, since $\Delta$ has finitely many eigenvalues in $[-\lambda, \lambda]$ (see Example \ref{ex:LaplaceOperator}). Hence, all norms are equivalent and instead of considering the operator norm topology we consider the strong topology.
For each $n \in \mathbb{N}$, we denote the equidistant partition of the unit interval into $n$ intervals of the same length by $\mathcal{P}_n = \{ I_1, I_2, \ldots, I_n \}$ with  $I_j = [ (j-1)/n, j/n ]$. 
Let $g$ be in the domain of $\Delta$ and $f := P_\Delta(\lambda) g \in PW_\Delta(\lambda)$, i.e.,
$$
f(x) = \sum_{k \in \mathbb{Z}: \, 4 \pi^2 k^2 \leq \lambda}  \langle f, e^{2 i \pi k (\cdot)} \rangle e^{2 i \pi k x}.
$$
Hence, $f$ is a smooth function on $[0,1]$, so
\begin{equation}
    \label{eq1:Proof}
    |f(x) - n \int_{x - 1/2n}^{x + 1/2n} f(y) dy| \xrightarrow{n \to \infty} 0.
\end{equation}
Let $k \in \mathbb{N}$. We have
\begin{equation}
    \label{eq2:Proof}
\max_{i = 1, \ldots, n} |(\Delta^k f)(x_i) - (\Delta_n^k f_n)_i | \xrightarrow{n \to \infty} 0,
\end{equation}
since the discretization of the Laplace operator  with finite difference approximations on equidistant grid points is consistent (see \cite[Chapter 2]{FiniteDifference}).
 Let $h(z) =  z^k \in \mathcal{H}$.
Since  $h(0) = 0$, Proposition \ref{prop:graphdomain1} leads to 
$
h(T_{\AA_n}) = T_{nh(\DD_n)}
$, 
i.e., $T^k_{\AA_n} = T_{n\DD_n^k}$. %
We denote the center of the interval $I_i := [i-1/n, i/n]$ by $x_i$, and calculate 
\begin{equation}
    \begin{aligned}
    T_{n\DD_n^k} (x_i) & = n \left ( \int_0^1 W_{\DD_n^k}(x_i, v) f(v) dv\right )   \\
    & = n \left ( \int_0^1 \sum_{g=1}^n \sum_{l=1}^n (\Delta^k_n)_{g,l} \chi_{I_g}(x_i) \chi_{I_l}(v) f(v) dv \right ) \\
    & = n \left ( \int_0^1  \sum_{l=1}^n (\Delta^k_n)_{i,l} \chi_{I_l}(v) f(v) dv \right ) 
    = n \sum_{l=1}^n (\Delta^k_n)_{i,l} \int_{I_l} f(v) dv.
    \end{aligned}
\end{equation}
This leads to
\begin{equation*}
\begin{aligned}
    \max_{i = 1, \ldots, n} | T_{\AA_n}^kf(x_i) -  (\Delta^k_n f_n)_i | & =
    \max_{ i = 1, \ldots, n} \Big| n \sum_{l=1}^n (\Delta^k_n)_{i,l} \int_{I_l} f(v) dv  - (\Delta^k_n)_{i,l} (f_n)_l  \Big| 
    \\
    & = \max_{ i = 1, \ldots, n}  \Big|   \sum_{l=1}^n (\Delta^k_n)_{i,l} \left(  n \int_{I_l} f(v) dv  -  (f_n)_l \right)  \Big|  \\
    & = \max_{ i = 1, \ldots, n}  \Big|   \sum_{l=1}^n (\Delta^k_n)_{i,l} \left (  n \int_{x_l - 1/2n}^{x_l + 1/2n} f(v) dv  -  f(x_l) \right ) \Big |.  
\end{aligned}
\end{equation*}
For the fixed $k \in \mathbb{N}$, the norms of the matrices $\Delta_n^k$ are uniformly bounded with bound independent of $n$. Hence, by (\ref{eq1:Proof}), we get 
\begin{equation}
\label{eq3:Proof}
    \max_{i = 1, \ldots, n} | T_{\AA_n}^kf(x_i) -  (\Delta^k_n f_n)_i | \xrightarrow{n \to \infty} 0.
\end{equation}
The convergence results (\ref{eq2:Proof}) and (\ref{eq3:Proof}) lead to
\begin{equation}
\label{eq4:Proof}
    \begin{aligned}
    & \max_{i = 1, \ldots, n} | \left( T_{\AA_n}^k f (x_i) - \Delta^k f(x_i)  \right) | \\
    & \leq \max_{i = 1, \ldots, n} | T_{\AA_n}^k f (x_i) - (\Delta_n^k f_n)_i  | + | (\Delta_n^k f_n)_i - \Delta^k f (x_i) | 
     \xrightarrow{ n \to \infty } 0.
    \end{aligned}
\end{equation}

Let $\varepsilon > 0$. Since $\Delta^k f$ is smooth and $\sum_{i=1}^n \chi_{I_i}(\cdot) \Delta^k f(x_i)$ is piecewise constant on the compact set $[0,1]$, there exists $j \in \mathbb{N}$ and $x'_n \in I_j = [x_j - 1/2n, x_j + 1/2n ]$ such that 
$$
\| \sum_{i=1}^n \chi_{I_i}(\cdot) \Delta^k f(x_i) - \Delta^k f\|_{L^\infty([0,1])} = |  \Delta^k f(x_j) - \Delta^k f(x'_n)|.
$$
Since $\Delta^k f $ is Lipschitz continuous, it follows
$$
|  \Delta^k f(x_j) - \Delta^k f(x'_n)| \leq  C |x_{j_n} - x'_{n}| \xrightarrow{n \to \infty} 0.
$$
Hence, we can choose  $n \in \mathbb{N}$ large enough such that
\begin{equation}
\| \sum_{i=1}^n \chi_{I_i}(\cdot) \Delta^k f(x_i) - \Delta^k f\|_{L^2[0,1]} < \varepsilon,
\end{equation}
and together with (\ref{eq4:Proof})
\begin{equation}
    \max_{i = 1, \ldots, n} |T_{\AA_n}^k f(x_i) - \Delta^k f (x_i)| < \varepsilon.
\end{equation}
As a result,
\begin{equation*}
\begin{aligned}
\| T_{\AA_n}^k f - \Delta^k f\|_{L^2[0,1]} &  = \| \sum_{i=1}^n \chi_{I_i}(\cdot)  T_{\AA_n}^k f (x_i) - \Delta^k f\|_{L^2[0,1]} \\
& \leq 
\|\sum_{i=1}^n \chi_{I_i}(\cdot)  T_{\AA_n}^k f (x_i) -  \sum_{i=1}^n \chi_{I_i}(\cdot) \Delta^k f(x_i) \|_{L^2[0,1]} \\ &
+
\| \sum_{i=1}^n \chi_{I_i}(\cdot) \Delta^k f(x_i) - \Delta^k f\|_{L^2[0,1]} \\
& \leq \|\sum_{i=1}^n \chi_{I_i}(\cdot)  T_{\AA_n}^k f (x_i) -  \sum_{i=1}^n \chi_{I_i}(\cdot) \Delta^k f(x_i) \|_{L^\infty([0,1])}
\\ & +
\| \sum_{i=1}^n \chi_{I_i}(\cdot) \Delta^k f(x_i) - \Delta^k f\|_{L^2[0,1]} \\
& =  \max_{i = 1, \ldots, n} |T_{\AA_n}^k f(x_i) - \Delta^k f (x_i)|
\\ & +
\| \sum_{i=1}^n \chi_{I_i}(\cdot) \Delta^k f(x_i) - \Delta^k P_\mathcal{L}(\lambda) f\|_{L^2[0,1]}
< 2 \varepsilon,
\end{aligned}    
\end{equation*}
i.e.,
\begin{equation*}
   \| P_\Delta(\lambda)T_{\AA_n}^k P_\Delta(\lambda)g - \Delta^k P_\Delta(\lambda) g \|_{L^2[0,1]} \leq \| T_{\AA_n}^k P_\Delta(\lambda)g - \Delta^k P_\Delta(\lambda)g\|_{L^2[0,1]} \xrightarrow{n \to \infty} 0.
\end{equation*}
Since $(G_n)_n \xrightarrow[n\to\infty]{U} \Delta$, Lemma \ref{thm:TransfUnbddL} leads to 
\begin{equation*}
\begin{aligned}
& \| P_\Delta(\lambda)T_{\AA_n}^k P_\Delta(\lambda)g - \Big( P_\Delta(\lambda)T_{\AA_n} P_\Delta(\lambda) \Big)^k g \| \\ & \leq
\| P_\Delta(\lambda)T_{\AA_n}^k P_\Delta(\lambda)g - \Delta^k P_\Delta(\lambda) g \|_{L^2[0,1]} \\ & + \| \Delta^k P_\Delta(\lambda) g -\Big( P_\Delta(\lambda)T_{\AA_n} P_\Delta(\lambda) \Big)^k g \|_{L^2[0,1]} 
\xrightarrow{n \to \infty} 0,
\end{aligned}
\end{equation*}
finishing the proof.
\end{proof}  

\subsection{Proof of Lemma~\ref{cl:asympTransferabilityGCNNon}}
\label{proof:asympTGCNNon}
For $l = 1$ and $j = 1, \ldots, F_1$, we have
\begin{equation}
\label{eq:GCNNoncalclayer}
\begin{aligned}
& \| (\tilde{\phi}^1_{W_{\pi_n(\AA_n)}}(y_n))^j  - (\tilde{\phi}^1_{W}(y))^j  \|_{L^2[0,1]} \\
& = \| \sum_{k=1}^{F_0} m_1^{jk}h_1^{jk}(T_{W_{\pi_n(\AA_n)}})(y_n^k) - \sum_{k=1}^{F_0} m_1^{jk}h_1^{jk}(T_W)(y^k)  \|_{L^2[0,1]} \\
& \leq     \sum_{k=1}^{F_0}  \| m_1^{jk}h_1^{jk}(T_{W_{\pi_n(\AA_n)}})(y_n^k) - m_1^{jk}h_1^{jk}(T_W)(y_n^k) \|_{L^2[0,1]} \\ & + \sum_{k=1}^{F_0} \|m_1^{jk} h_1^{jk}(T_W)(y_n^k) - m_1^{jk}h_1^{jk}(T_W)(y^k)  \|_{L^2[0,1]} 
\\
& \leq     \sum_{k=1}^{F_0}   |m_1^{jk}|\|h_1^{jk}(T_{W_{\pi_n(\AA_n)}})(y_n^k) - h_1^{jk}(T_W)(y_n^k) \|_{L^2[0,1]} \\ & + \sum_{k=1}^{F_0} |m_1^{jk}| \| h_1^{jk}(T_W)(y_n^k) - h_1^{jk}(T_W)(y^k)  \|_{L^2[0,1]} 
\end{aligned}
\end{equation}
By Lemma~\ref{claim:ConvOfGrFltBdd} and by (\ref{eq1:cl:asym}), every summand in the-right-hand-side of (\ref{eq:GCNNoncalclayer}) converges to zero as $n \rightarrow \infty$. Since $\rho: \mathbb{R} \rightarrow \mathbb{R}$ is Lipschitz continuous, we have
\begin{equation*}
\begin{aligned}
       & \int_0^1 |\rho((\tilde{\phi}^1_{W_{\pi_n(\AA_n)}}(y_n))^j (x) )- \rho((\tilde{\phi}^1_{W}(y))^j (x))|^2dx  \\ & \leq \Lambda^2 \int_0^1 |(\tilde{\phi}^1_{W_{\pi_n(\AA_n)}} (y_n))^j - (\tilde{\phi}^1_{W}(y))^j |^2dx ,
\end{aligned}
\end{equation*}
where  $\Lambda \in \mathbb{R}$ is the Lipschitz constant of $\rho$.  Hence, for every feature $j = 1, \ldots,  F_1$ , we have $\|(\phi^1_{W_{\pi_n(\AA_n)}}(y_n))^j  - (\phi^1_{W}(y))^j  \|_{L^2[0,1]} \xrightarrow{n \to \infty} 0$. 
By induction over the layers, we have
\begin{equation*}
 \|(\phi^L_{W_{\pi_n(\AA_n)}}(y_n))^j  - (\phi^L_{W}(\tilde{\phi}^{L-1}_{W}(y)))^j \|_{L^2[0,1]} \xrightarrow{n\to\infty} 0
\end{equation*}
for all features $j \in \{1,\ldots,F_L\}$. 

\subsection{Proof of Lemma~\ref{claim:ConvOfGrNN}}
\label{proof:LinConvGCNN}
 For every $l= 1, \ldots, L$, let $C_l^{jk}$ denote the constant $C$ of (\ref{eq:LinConv}) corresponding to the filter $h=h_l^{jk}$, for $j = 1, \ldots, F_l$, $k = 1, \ldots, F_{l-1}$.  Let $C=\max_{l,j,k} C_l^{jk}$, where the maximum is over $l = 1,\ldots, L$, $j = 1, \ldots, F_l$, and $k = 1, \ldots, F_{l-1}$.
Further, let $M = \max_{l = 1, \ldots,L} \|\mathbf{M}_l\|_\infty$.

For $j = 1, \ldots, F_L$, we have
\begin{equation}
\label{eq:estimForLayerOutput}
\begin{aligned}
& \| (\tilde{\phi}^L_{W_{\pi(\AA)}}(\psi_\mathbf{x}))^j - (\tilde{\phi}^L_{W} (y) )^j \|_{L^2[0,1]} \\
 & \leq \sum_{k=1}^{F_{L-1}}  \| m_L^{jk}h_L^{jk}(T_{W_{\pi(\AA)}})(\phi^{L-1}_{W_{\AA}}(\psi_\mathbf{x})^k) - m_L^{jk}h_L^{jk}(T_W)(\phi^{L-1}_{W_{\AA}}(\psi_\mathbf{x})^k) \|_{L^2[0,1]} \\ & + \sum_{k=1}^{F_{L-1}} \| m_L^{jk} h_L^{jk}(T_W)(\phi^{L-1}_{W_{\AA}}(\psi_\mathbf{x})^k) - m_L^{jk}h_L^{jk}(T_W)(\phi^{L-1}_{W}(y)^k)  \|_{L^2[0,1]} .
\end{aligned}    
\end{equation}
For the first summand on the right hand side of~(\ref{eq:estimForLayerOutput}), we have
\begin{equation}
\label{eq2:ProofOfGCNNLinConv}
\begin{aligned}
 & \sum_{k=1}^{F_{L-1}}  \| m_L^{jk}h_L^{jk}(T_{W_{\pi(\AA)}})(\phi^{L-1}_{W_{\AA}}(\psi_\mathbf{x})^k) - m_L^{jk}h_L^{jk}(T_W)(\phi^{L-1}_{W_{\AA}}(\psi_\mathbf{x})^k) \|_{L^2[0,1]} \\
 & \leq \sum_{k=1}^{F_{L-1}} |m_L^{jk}| \| h_L^{jk}(T_{W_{\pi(\AA)}})- h_L^{jk}(T_W) \| \|\phi^{L-1}_{W_{\AA}}(\psi_\mathbf{x})^k \|_{L^2[0,1]} \\
 &  \leq C \varepsilon \sum_{k=1}^{F_{L-1}} |m_L^{jk}|  \|\phi^{L-1}_{W_{\AA}}(\psi_\mathbf{x})^k \|_{L^2[0,1]} \leq \varepsilon C M \max_{k = 1, \ldots, F_{L-1}} \|\phi^{L-1}_{W_{\AA}}(\psi_\mathbf{x})^k\|_{L^2[0,1]}.
\end{aligned}
\end{equation}
For the second summand in~(\ref{eq:estimForLayerOutput}), we calculate
\begin{equation}
\label{eq1:ProofOfGCNNLinConv}
\begin{aligned}
 & \sum_{k=1}^{F_{L-1}} \| m_L^{jk}h_L^{jk}(T_W)(\phi^{L-1}_{W_{\AA}}(\psi_\mathbf{x})^k) - m_L^{jk}h_L^{jk}(T_W)(\phi^{L-1}_{W}(y)^k )  \|_{L^2[0,1]}  \\
  &  \leq \sum_{k=1}^{F_{L-1}} \| m_L^{jk}h_L^{jk}(T_W) \| \|\phi^{L-1}_{W_{\AA}}(\psi_\mathbf{x})^k - \phi^{L-1}_{W}(y)^k  \|_{L^2[0,1]}  \\
  & \leq \sum_{k=1}^{F_{L-1}} | m_L^{jk}| \|h_L^{jk} \|_{\infty} \|\phi^{L-1}_{W_{\AA}}(\psi_\mathbf{x})^k - \phi^{L-1}_{W}(y)^k  \|_{L^2[0,1]}  \leq M \max_k \mathcal{R}_{L-1}^k,
\end{aligned}
\end{equation}
where $\mathcal{R}_{L-1}^k := \|\phi^{L-1}_{W_{\AA}}(\psi_\mathbf{x})^k - \phi^{L-1}_{W}(y)^k  \|_{L^2[0,1]}$.
Together, (\ref{eq1:ProofOfGCNNLinConv}) and (\ref{eq2:ProofOfGCNNLinConv}) lead to
$$
\begin{aligned}
& \| (\tilde{\phi}^L_{W_{\pi(\pi(\AA))}} (\psi_\mathbf{x}))^j - (\tilde{\phi}^L_{W}(y) )^j  \|_{L^2[0,1]} \\ & \leq M  \max_{k = 1, \ldots, F_{L-1}}\mathcal{R}_{L-1}^k   + C \varepsilon M \max_{k = 1, \ldots, F_{L-1}} \|\phi^{L-1}_{W_{\AA}}(\psi_\mathbf{x})^k\|_{L^2[0,1]}
\end{aligned}
$$
for every $j = 1, \ldots, F_L$. Then,
\begin{equation}
\label{eq:RecRel}
\begin{aligned}
& \| (\phi^L_{W_{\pi(\pi(\AA))}})^j (\psi_\mathbf{x}) - (\phi^L_{W})^j (y)   \|_{L^2[0,1]}  \\ & \leq  M  \max_{k = 1, \ldots, F_{L-1}}\mathcal{R}_{L-1}^k   + C\varepsilon M \max_{k = 1, \ldots, F_{L-1}} \|\phi^{L-1}_{W_{\AA}}(\psi_\mathbf{x})^k\|_{L^2[0,1]},
\end{aligned}
\end{equation}
since the activation function is contractive.
We solve the recurrence relation in (\ref{eq:RecRel}). For this, we estimate $\max_{k = 1, \ldots, F_l} \|\phi^l_{W_{\AA}}(\psi_\mathbf{x})^k\|_{L^2[0,1]}$ for $l = 1, \ldots, L$. 
Now, for $k = 1, \ldots, F_l$,
\[
\begin{aligned}
\|\phi^l_{W_{\AA}}(\psi_\mathbf{x})^k\|_{L^2[0,1]} & = 
\sum_{k'=1}^{F_{l-1}}  \| m_l^{jk'}h_l^{jk'}(T_{W_{\pi(\AA)}})(\phi^{l-1}_{W_{\AA}}(\psi_\mathbf{x})^{k'}) \|_{L^2[0,1]} \\
& \leq M  \max_{k' = 1, \ldots, F_{l-1}} \|\phi^{l-1}_{W_{\AA}}(\psi_\mathbf{x})^{k'}\|_{L^2[0,1]}.
\end{aligned}
\]
Hence,
\[
\|\phi^L_{W_{\AA}}(\psi_\mathbf{x})^k\|_{L^2[0,1]} \leq  M^L \max_{k' = 1, \ldots, F_{l-1}}\|\psi_\mathbf{x}^{k'}\|_{L^2[0,1]}. 
\]
This leads to
\[
\begin{aligned}
 & \| (\phi^L_{W_{\pi(\pi(\AA))}} (\psi_\mathbf{x}) )^j- (\phi^L_{W} (y) )^j  \|_{L^2[0,1]} \leq M^L(1 + L C) \max_{k = 1, \ldots, F_0} \| \psi_\mathbf{x}^k - y^k \|_{L^2[0,1]}.
\end{aligned}
\]

\section*{Acknowledgments}

S.M acknowledges partial support by the NSF–Simons Research Collaboration on the Mathematical and Scientific Foundations of Deep Learning (MoDL) (NSF DMS
2031985) and DFG SPP 1798, KU 1446/27-2.

R.L. acknowledges support by the DFG SPP 1798, KU 1446/21-2 “Compressed Sensing in Information Processing” through Project Massive MIMO-II.

G.K. acknowledges partial support by the NSF–Simons Research Collaboration on the Mathematical and Scientific Foundations of Deep Learning (MoDL) (NSF DMS
2031985) and DFG SPP 1798, KU 1446/27-2 and KU 1446/21-2. 
\bibliographystyle{siam}
\bibliography{references}
\end{document}